\documentclass[11pt]{article}

\usepackage{jmlr2e}

\usepackage{lipsum}
\usepackage{amsfonts}
\usepackage{graphicx}
\usepackage{epstopdf}
\usepackage{algorithmic}
\ifpdf
  \DeclareGraphicsExtensions{.eps,.pdf,.png,.jpg}
\else
  \DeclareGraphicsExtensions{.eps}
\fi

\usepackage{enumitem}
\setlist[enumerate]{leftmargin=.5in}
\setlist[itemize]{leftmargin=.5in}

\usepackage{amsopn}
\DeclareMathOperator{\diag}{diag}

 \usepackage{amsmath,amssymb}
\usepackage{graphicx}
\usepackage{caption}
\usepackage{subcaption}
\usepackage{xcolor}

\newcommand{\mb}{\mathbb}
\newcommand{\mc}{\mathcal}
\newcommand{\rar}{\rightarrow}

\newcommand{\pxtwo}[2]{x_{#1,#2}}

\newcommand{\pxone}[1]{x_{#1}}

\newcommand{\pxt}[1]{x_{#1}}

\newcommand{\D}{D}
\newcommand{\bmat}[1]{\begin{bmatrix}#1\end{bmatrix}}
\newcommand{\spec}{\mathrm{spec}}

\newcommand{\f}{f}

 \newcommand{\tr}{\mathrm{tr}}
\newcommand{\dd}{\mathrm{d}}
\usepackage{multirow}
 
 \usepackage{enumitem}
 \newtheorem{assumption}{Assumption}

 \newcommand{\LNE}{{\tt{LNE}}}
 \newcommand{\LASE}{{\tt{LASE}}}
\newcommand{\SSP}{{\tt{SSP}}}
\newcommand{\NDDNE}{{\tt{NDDNE}}}
\newcommand{\DNE}{{\tt{DNE}}}

\newcommand{\nplayers}{n}
\newcommand{\dimsmth}{s}
\usepackage{times}
\newif\ifsecs
\secsfalse
\newif\ifzerosum
\zerosumfalse

\newif\ifms
\msfalse

\ifpdf
\hypersetup{
  pdftitle={On Gradient-Based Learning in Continuous Games},
  pdfauthor={}
}
\fi

\ShortHeadings{On Gradient-Based Learning in Continuous Games}{Mazumdar, Ratliff, and Sastry}
\firstpageno{1}

\begin{document}

\title{On Gradient-Based Learning in Continuous Games \thanks{This work was published in the SIAM Journal on Mathematics for Data Science on February 18, 2020 (https://doi.org/10.1137/18M1231298). This work was funded by the National Science Foundation Award
CNS:1656873 and the Defense Advanced Research Projects Agency Award FA8750-18-C-0101}}
\author{\name Eric Mazumdar \email emazumdar@eecs.berkeley.edu \\ \addr Department of Electrical Engineering and Computer Science\\
       University of California, Berkeley, CA 
       \AND
       \name Lillian J. Ratliff \email ratliffl@uw.edu \\ \addr Department of Electrical and Computer Engineering\\
       University of Washington, Seattle, WA 
       \AND
       \name S. Shankar Sastry \email sastry@coe.berkeley.edu \\ \addr Department of Electrical Engineering and Computer Science\\
       University of California, Berkeley, CA }

\editor{}

\maketitle

\begin{abstract}

We introduce a general framework for competitive gradient-based learning 
that encompasses a wide breadth of multi-agent learning algorithms, and analyze the limiting behavior of competitive gradient-based learning algorithms using dynamical systems theory. For both general-sum and potential games, we characterize  a
non-negligible subset of the local Nash equilibria that will be avoided if each agent employs a gradient-based learning algorithm. We also shed light on the issue of convergence to non-Nash strategies in general- and zero-sum games, which may have no relevance to the underlying game, and arise solely due
to the choice of algorithm. The existence and frequency of such strategies may
explain some of the difficulties encountered when using gradient descent in
zero-sum games as, e.g., in the training of generative adversarial networks. To reinforce the theoretical
contributions, we provide empirical results that highlight the frequency of linear quadratic dynamic games (a benchmark for multi-agent reinforcement learning) that admit global Nash equilibria that are almost surely avoided by policy gradient. %

 \end{abstract}


\section{Introduction}

 With machine learning algorithms increasingly being deployed in real world settings, it is crucial that we understand how the algorithms can interact, and the dynamics that can arise from their interactions. In recent years, there has been a resurgence in research efforts on multi-agent learning, and learning in games. The recent interest in adversarial learning techniques also serves to show how game theoretic tools can be being used to \emph{robustify} and improve the performance of machine learning algorithms. Despite this activity, however, machine learning algorithms are still being treated as black-box approaches and being na\"{i}vely deployed in settings where other algorithms are actively changing the environment. In general, outside of highly structured settings, there exists no guarantees on the performance or limiting behaviors of learning algorithms in such settings. 

Indeed, previous work on understanding the collective behavior of coupled learning algorithms, either in competitive or cooperative settings,
has mainly looked at games where the global structure is well understood like  bilinear games\cite{GradDyn,hommes:2012aa,mertikopoulos:2018aa,leslie:2005aa}, convex games \cite{Mertikopoulos2019,Rosen1965}, or potential games \cite{monderer:1996aa}, among many others. Such games are more conducive to the statement of global convergence guarantees since the assumed global structure can be exploited.  

In games with fewer assumptions on the players' costs, however, there is still a lack of understanding of the dynamics and limiting behaviors of learning algorithms. Such settings are becoming increasingly prevalent as deep learning is increasingly being used in game theoretic settings \cite{goodfellow:2014aa,foerster:2017aa,abdallah:2008aa,zhang:2010aa}.

Gradient-based learning algorithms are extremely popular in a variety of these multi-agent settings due to their versatility, ease of implementation, and dependence on local information. There are numerous recent papers in multi-agent
reinforcement learning that employ  gradient-based methods (see,
e.g.\cite{abdallah:2008aa, foerster:2017aa,zhang:2010aa}), yet even within this well-studied class of learning algorithms, a thorough
understanding of their convergence and limiting behaviors in general continuous games is still lacking. 

Generally speaking, in both the game theory and the machine learning communities,  two of the central questions when analyzing the dynamics of learning in games are the following:
\begin{description}[leftmargin=20pt]
    \item[Q1.] \emph{Are all attractors of the learning algorithms employed by agents equilibria relevant to the underlying game?}
    \item[Q2.] \emph{Are all equilibria relevant to the game also attractors of the learning algorithms agents employ?}
\end{description}
In this paper, we provide some answers to the above questions for the class of gradient-based learning algorithms by analyzing their limiting behavior in general continuous games. 
In particular, we leverage the continuous time limit of the more naturally discrete multi-agent learning algorithms.
This allows us to draw on the extensive theory of dynamical systems and stochastic approximation to make statements about the limiting behaviors of these algorithms in both deterministic and stochastic settings. The latter is particularly relevant since it is common for stochastic gradient methods to be used in multi-agent machine learning contexts.

Analyzing gradient-based algorithms through the lens of dynamical systems theory has recently yielded new insights into their behavior in the classical optimization setting \cite{Wilson2016,Bach,lee:2016aa}. We show that a similar type of analysis can also help understand the limiting behaviors of gradient-based algorithms in games.  
We remark, however, that there is a \emph{fundamental difference} between the dynamics that are analyzed in much of the single-agent, gradient-based learning and optimization literature and the ones we analyze in the
competitive multi-agent case: the combined dynamics of gradient-based learning schemes in games \emph{do not necessarily correspond to a gradient flow}.
This may seem a subtle point, but it
it turns out to be extremely important. 

Gradient flows admit desirable convergence guarantees---e.g., almost sure
convergence to local minimizers---due to the fact that they
preclude flows with the \emph{worst geometries}~\cite{pemantle:2007aa}.
In particular, they do not exhibit non-equilibrium limiting behavior such as
periodic orbits. Gradient-based learning in games, on the other hand, does not
preclude such behavior. Moreover, as we show, asymmetry in the dynamics of gradient-play in games can lead to surprising behaviors such as non-relevant limiting behaviors being
attracting under the flow of the game dynamics and relevant limiting
behaviors, such as a subset of the Nash equilibria being almost surely avoided. 

\subsection{Related Work}
The study of continuous games is quite extensive~(see e.g. \cite{BasarOlsder, osborne:1994aa}), though in large part the focus has been on games admitting a fair amount of structure. The behavior of learning algorithms in games is also well-studied~(see e.g. \cite{fudenberg:1998aa}). In this section, we comment on the most relevant prior work and defer a more comprehensive discussion of our results in the context of prior work to Section~\ref{sec:discussion}.

As we noted, previous work on learning in games in both the game theory literature, and more recently from the machine learning community,
has largely focused on addressing (\textbf{Q1}) whether all attractors of the learning dynamics are game-relevant equilibria,  and (\textbf{Q2}) whether all game-relevant equilibria are also attractors of the learning dynamics. The primary type of game-relevant equilibrium considered in the investigation of these two questions is a Nash equilibrium. 

The majority of the existing work  has focused on \textbf{Q1}. 
In fact, a large body of prior work focuses on games with structures that preclude the existence of non-Nash equilibria. Consequently, answering \textbf{Q1} reduces to analyzing the convergence of various learning algorithms (including gradient-play) to the unique Nash equilibrium or the set of Nash equilibria. This is often shown by exploiting the game structure. Examples of classes of games where gradient-play has been well-studied are potential games~\cite{monderer:1996aa}, concave or monotone games \cite{Rosen1965,bravo2018bandit,Mertikopoulos2019}, and gradient-play over the space of stochastic policies in two-player finite-action bilinear games \cite{GradDyn}. In the latter setting, other gradient-like algorithms such as multiplicative weights have also been studied fairly extensively \cite{hommes:2012aa}, and have been shown to converge to cycling behaviors.

Some works have also attempted to address \textbf{Q1} in the context of gradient-play in two-player zero-sum games. Concurrently with this paper, for a general class of ``sufficiently smooth'' two-player, zero-sum games it was shown that there exists stationary points for gradient-play that are non-Nash \cite{Daskalakis}\footnote{This paper was under review at the time that \cite{Daskalakis} became publicly available. Our results show the existence of these non-Nash equilibria and attracting cycles in both general-sum and zero-sum games.}.
In such games, it has also been shown that gradient-play can converge to cycles (see, e.g.,~\cite{mertikopoulos:2018aa, Wesson2016, hommes:2012aa}).

There is also related work in more general games on the analysis of when Nash equilibria are attracting for gradient-based approaches (i.e. \textbf{Q2}). Sufficient conditions for this to occur are the conditions for stable differential Nash equilibria  introduced in \cite{ratliff:2013aa,ratliff:2014aa,ratliff:2016aa} and the condition for variational stability later analyzed in \cite{Mertikopoulos2019}. We remark that these conditions are equivalent for the classes of games we consider. Neither of these works give conditions under which Nash equilibria are avoided by gradient-play or comment on other attracting behaviors.

Expanding on this rich body of literature (only the most relevant of which is covered in our short review), in this paper we provide answers to \textbf{Q1} without imposing structure on the game outside regularity conditions on the cost functions by exploiting the observation that gradient-based learning dynamics are not gradient flows. We also provide answers to \textbf{Q2} by demonstrating that a non-trivial set of games admit Nash equilibria that are almost surely avoided by gradient-play. We give explicit conditions for when this occurs.
Using similar analysis tools, we also provide new insights into the behavior of gradient-based learning in structured classes of games such as zero-sum and potential games.

\subsection{Contributions and Organization}
We present a general framework for modeling competitive gradient-based learning
that applies to a broad swath of learning algorithms. 
In Section~\ref{sec:connections}, we draw connections between the limiting behavior of this class of algorithms and game-theoretic and dynamical systems notions of equilibria. In particular, we construct general-sum and zeros-sum games that admit non-Nash attracting equilibria of the gradient dynamics. Such points are attracting under the learning dynamics, yet at least one player---\emph{and potentially all of them}---has a direction in which they could unilaterally deviate to decrease their cost. Thus, these non-Nash equilibria are of questionable game theoretic relevance and can be seen as artifacts of the players' algorithms.

In Section~\ref{sec:results}, we show that policy gradient multi-agent reinforcement learning (MARL), generative adversarial networks (GANs), gradient-based multi-agent multi-armed bandits, among several other common multi-agent learning settings, conform to this framework.  The framework is amenable to  tools for analysis from dynamical systems theory. 

Also in Section~\ref{sec:results}, we show that a subset of the local Nash equilibria
in general-sum games and potential games is avoided almost surely when each
player employs a gradient-based algorithm. We show that this holds in two broad
settings: the full information setting when each player has oracle access to
their gradient but randomly initializes their first action, and a partial
information setting where each player has access to an unbiased estimate of
their gradient. 

Thus, we provide a negative answer to both  \textbf{Q1} and  \textbf{Q2} for $\nplayers$--player general-sum games, and highlight the nuances present in  zero-sum and potential games. We also show that the dynamics formed
        from the individual gradients of agents' costs  are \emph{not gradient
        flows}. This in turn implies that competitive gradient-based learning in general-sum games may converge to periodic orbits and other non-trivial limiting behaviors that arise in, e.g., chaotic systems.
        
To support the theoretical results, we present empirical results in Section~\ref{sec:LQR} that show that policy gradient algorithms avoid
 global Nash equilibria in a large number of linear quadratic (LQ)
dynamic games, a benchmark for MARL. 

    We conclude in Section~\ref{sec:discussion} with a discussion of the implications of our results and some links with prior work as well as some comments on future directions.  

 \section{Preliminaries}
\label{sec:prelims}

Consider $\nplayers$ agents indexed by $\mc{I}=\{1, \ldots, \nplayers\}$. Each agent $i \in
\mc{I}$ has their own decision variable $x_i \in X_i$, where $X_i$ is their
finite-dimensional strategy space of dimension $m_i$. 
Define $X=X_1\times \cdots \times X_\nplayers$ to be the finite-dimensional joint strategy
space with dimension $m=\sum_{i\in \mc{I}}m_i$. Each agent is endowed with a cost function $\f_i\in C^\dimsmth(X, \mathbb{R})$ with
 $\dimsmth\geq 2$ and such that
 $\f_i:(\pxone{i},\pxone{-i})\mapsto f_i(x_i,x_{-i})$ where we use the notation
 $x=(x_i,x_{-i})$ to make the dependence on the action of the agent $x_i$, and the actions of all agents excluding agent $i$, $\pxone{-i}=(\pxone{1}, \ldots, \pxone{i-1}, \pxone{i+1}, \ldots,
\pxone{n})$ explicit. The agents seek to minimize their own cost, but only have 
control over their own decision variable $x_i$. In this setup,
agents' costs are not necessarily aligned with one another, meaning they are competing. 

Given the game $\mc{G}=(f_1,\ldots, f_\nplayers)$, 
agents are assumed to update their strategies simultaneously according to a gradient-based learning algorithm of the form
\begin{equation}
    x_{i,t+1}=x_{i,t}-\gamma_{i,t}h_i(x_{i,t}, x_{-i,t}),
    \label{eq:gradbasedlearn}
\end{equation}
where $\gamma_{i,t}$ is agent $i$'s step-size at iteration $t$. 

We analyze the following two settings: 
\begin{enumerate}
    \item Agents have \emph{oracle access} to the 
gradient of their cost with respect to their own choice
variable---i.e.~$h_i(x_{i,t},x_{-i,t})= D_if_i(x_{i,t},x_{-i,t})$ where $D_if_i\equiv \partial f_i/\partial x_i$
denotes the derivative of $f_i$ with respect to $x_i$.
\item Agents have an \emph{unbiased estimator} of their
gradient---i.e.,~$h_i(x_{i,t},x_{-i,t})=D_if_i(x_{i,t},x_{-i,t})+w_{i,t+1}$ where
$\{w_{i,t}\}$ is a zero mean, finite variance stochastic process.
\end{enumerate}
 We
refer to the former setting as \emph{deterministic} gradient-based learning and the latter setting as \emph{stochastic} gradient-based learning. 
Assuming that all agents are employing such algorithms, we aim to analyze the limiting behavior of the agents' strategies. To do so, we leverage the following game-theoretic notion of a Nash equilibrium.

\begin{definition}
  \label{def:SLNE}
  A strategy $x\in X$ is a {local Nash equilibrium}
  for the game $(f_1, \ldots, f_\nplayers)$ if, for each $i\in\mc{I}$, there exists
  an open set $W_i\subset X_i$ such that
  that $\pxone{i}\in W_i$ and 
  $f_i(x_i,x_{-i})\leq f_i(x_i',x_{-i})$
  for all $\pxone{i}'\in W_i$.
    If the above inequalities are strict, then we say
    $x$ is a {strict local Nash equilibrium}.  
\end{definition}

The focus on \emph{local} Nash equilibria is due to our lack of assumptions on the agents' cost functions. If $W_i=X_i$ for each $i$, then a local Nash equilibrium $x$ is a {global Nash equilibrium}. This holds in e.g the bimatrix games and the linear quadratic games we analyze in Section~\ref{sec:LQR}. Depending on the agents' costs, a game $(f_1, \ldots, f_\nplayers)$ may admit anywhere from one to a continuum of local or global Nash equilibria; or none at all.

 \section{Linking Games and Dynamical Systems}
\label{sec:connections}
   In this section, we draw links between the limiting behavior of dynamical
   systems and game-theoretic notions of equilibria in three broad classes of continuous games. For brevity, the proofs of the propositions in this section are supplied in Appendix~\ref{app:proofs}. A high-level summary of the links we draw is shown in Figure~\ref{fig:eqtype}.

Define $\omega(x)=(D_1f_1(x), \ldots, D_\nplayers f_\nplayers(x))$ to be the vector of player derivatives of their own cost functions with respect to their own choice variables. When each player is employing a gradient-based learning algorithm, the joint strategy of the players, (in the limit as the agents' step-sizes go to zero) follows the differential equation
\begin{align*}
    \dot x=-\omega(x).
    \label{eq:sys}
\end{align*}

A point $x\in X$ is said to be an equilibrium, critical point, or stationary point of the dynamics if $\omega(x)=0$. Stationary points of $\dot x=-\omega(x)$ are joint strategies from which, under gradient-play, the agents do not move. We note that $\omega(x)=0$ is a necessary condition for a point $x\in X$ to be a local Nash equilibrium~\cite{ratliff:2016aa}. Hence, all local Nash equilibria are critical points of the joint dynamics $\dot x=-\omega(x)$.

Central to dynamical systems theory is the study of limiting behavior and its
stability properties. A classical result in dynamical systems theory allows us to characterize the stability properties of an equilibrium $x^*$ by analyzing the Jacobian of the dynamics at $x^*$. The Jacobian of $\omega$ is defined by
\[D\omega(x)=\bmat{D_{1}^2\f_1(x) & \cdots & D_{\nplayers 1}\f_1(x)\\ \vdots & \ddots
     & \vdots \\ D_{1\nplayers}\f_\nplayers(x) & \cdots & D_{\nplayers}^2\f_\nplayers(x)}.\]
Since $D\omega$ is a matrix of second derivatives, it is sometimes referred to as the `game Hessian'. Similar to the Hessian matrix of a gradient flow, $D\omega$  allows us to further characterize the critical points of $\omega$ by their properties under the flow of $\dot x=-\omega(x)$. 
Let $\lambda_i(x)\in \spec(D\omega(x))$ for $i\in \{1, \ldots, m\}$ denote the
eigenvalues of $D\omega$ at $x$ where $\text{Re}(\lambda_1(x))\leq \cdots \leq
\text{Re}(\lambda_m(x))$---that is, $\lambda_1(x)$ is the eigenvalue with the
smallest real part. Of particular interest are asymptotically stable equilibria. 

\begin{definition}
A  point $x\in X$ is a {locally asymptotically stable equilibrium} of the continuous time dynamics $\dot x=-\omega(x)$ if $\omega(x)=0$ and $\mathrm{Re}(\lambda)>0$ for all $\lambda\in \spec(D\omega(x))$.    
\end{definition}

Locally asymptotically stable equilibria have two properties of interest. First, they are isolated, meaning that there exists a neighborhood around them in which no other equilibria exist. Second, they are exponentially attracting under the flow of $\dot x=-\omega(x)$, meaning that if agents initialize in a neighborhood of a locally asymptotically stable equilibrium $x^\ast$ and follow the dynamics described by $\dot x=-\omega(x)$, they will converge to $x^\ast$ exponentially fast \cite{sastry:1999aa}. This, in turn, implies that a discretized version of $\dot x=-\omega(x)$, namely
\begin{equation}\pxt{t+1}=\pxt{t}-\gamma\omega(\pxt{t}),\label{eq:update-x}\end{equation} converges locally for appropriately selected step size $\gamma$ at a rate of $O(1/t)$. Such results motivate the study of the continuous time
dynamical system $\dot{x}=-\omega(x)$ in order to understand convergence
properties of gradient-based learning algorithms of the form
\eqref{eq:gradbasedlearn}.

Another important class of critical points of a dynamical system are saddle points.
\begin{definition}
    A point $x\in X$ is a {saddle point} of the dynamics $\dot x=-\omega(x)$ if $\omega(x)=0$ and $\lambda_1(x)\in \spec(D\omega(x))$ is such that
    $\mathrm{Re}(\lambda_1(x))\leq 0$.     
    A saddle point such that $\mathrm{Re}(\lambda_i)<0$ for $i\in \{1, \ldots, \ell\}$ and $\mathrm{Re}(\lambda_j)>0$ for $j\in\{\ell+1, \ldots, m\}$ with $0<\ell<m$ is a {strict saddle point} of the continuous time dynamics $\dot x=-\omega(x)$.
\end{definition}

Strict saddle points are especially relevant to our analysis since their neighborhoods are characterized by stable and unstable manifolds \cite{sastry:1999aa}. When the agents evolve according to the dynamics solely on the stable manifold, they converge exponentially fast to the critical point. However, when they evolve solely on the unstable manifold, they diverge from the equilibrium exponentially fast. Agents whose strategies lie on the union of the two manifolds asymptotically avoid the equilibrium. We make use of this general fact in Section~\ref{sec:fullinfo}.

To better understand the links between the critical points of the gradient dynamics and the Nash equilibria of the game, we make use of an equivalent characterization of strict local Nash that leverages first and second order conditions on player cost functions. This makes them simpler objects to link to the various dynamical systems notions of equilibria than local Nash equilibria.

\begin{definition}[\cite{ratliff:2013aa, ratliff:2016aa}]
    A point $x\in X$ is a differential Nash equilibrium for the game
    defined by $(f_1,\ldots, f_\nplayers)$ if  $\omega(x)=0$ and $D_{i}^2f_i(x)
    \succ 0$ for each
    $i\in \mc{I}$.
\end{definition}

In \cite{ratliff:2014aa}, it was shown that local Nash equilibria are generically differential Nash equilibria where $\det(D\omega(x))\neq 0$ (i.e., $D\omega$ is non-degenerate). Thus, in the space of games where the agents' costs are at least twice differentiable, the set of games that admit local Nash equilibria that are not non-degenerate differential Nash equilibria is of measure zero \cite{ratliff:2014aa}.  In \cite{ratliff:2014aa} it was also shown that non-degenerate Nash equilibria are structurally stable, meaning that small perturbations to the agents' costs functions will not change the fundamental nature of the equilibrium. This also implies that gradient-play with slightly biased estimators of the gradient will not have vastly different behaviors in neighborhoods of equilibria.

Given these different equilibrium notions of the learning dynamics and the underlying game, let us define the following sets which will be useful in stating the results in the following sections.  For a game $\mc{G}=(f_1,\ldots, f_\nplayers)$, denote the sets of strict saddle points and locally asymptotically stable equilibria of the gradient
dynamics, $\dot x=-\omega(x)$, as $\SSP(\omega)$ and $\LASE(\omega)$, respectively, where we
recall that $\omega(x)=(D_1f_1(x), \ldots, D_\nplayers f_\nplayers(x))$. Similarly,
denote the set of local Nash equilibria, differential Nash equilibria, and
non-degenerate differential Nash equilibria of $\mc{G}$ as
$\LNE(\mc{G})$, $\DNE(\mc{G})$, and $\NDDNE(\mc{G})$,
respectively. As previously mentioned, $\NDDNE(\mc{G})=\LNE(\mc{G})$ in almost all continuous games. The key takeaways of this section are summarized in Figure~\ref{fig:eqtype}. 

\begin{figure}[h]
\center    
      \includegraphics[width=0.8\linewidth]{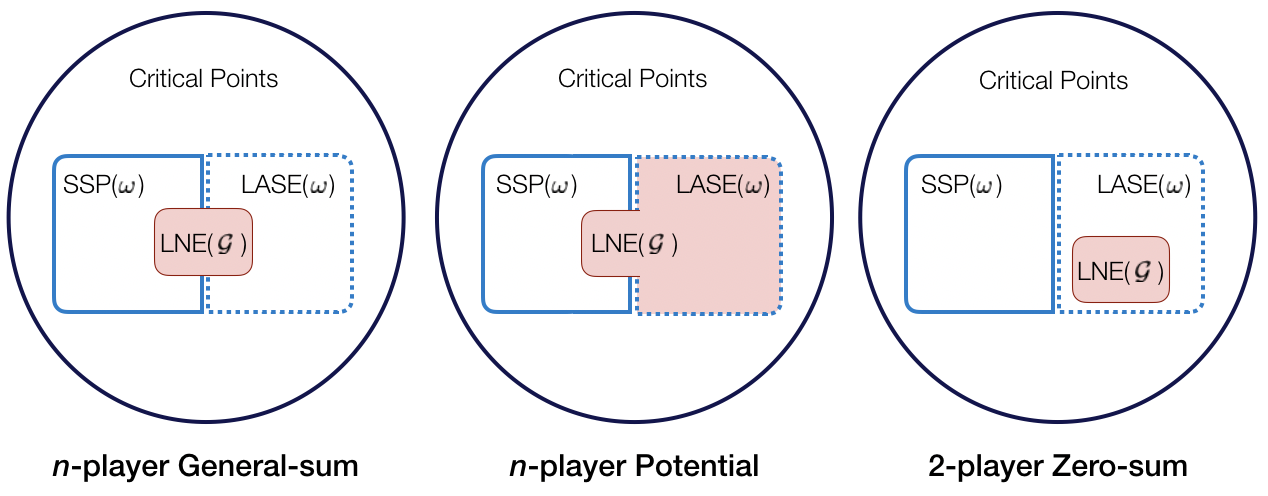}
        \caption{Links between the equilibria of generic continuous games $\mc{G}$ and their properties under the gradient dynamics $\dot x=-\omega(x)$.}
  \label{fig:eqtype}
\end{figure}

\subsection{General-sum games}
We first analyze the properties of local Nash equilibria under the joint gradient dynamics in $\nplayers$-player general-sum games.

\begin{proposition}
A non-degenerate differential Nash equilibrium is either a locally asymptotically stable equilibrium or a strict saddle point of $\dot x=-\omega(x)$---i.e., $\NDDNE(\mc{G})\subset \SSP(\omega)\cup \LASE(\omega)$.
\label{lem:nddne}
\end{proposition}
Locally asymptotically stable differential Nash equilibria satisfy the notion of variational stability introduced in \cite{Mertikopoulos2019}. In fact, a simple analysis shows that the definitions of variationally stable equilibria and locally asymptotically stable differential Nash equilibria~\cite{ratliff:2013aa} are equivalent in the games we consider---i.e., games where each players' cost is at least twice continuously differentiable. 
We remark that, from the definition of asymptotic stability, the gradient dynamics have a $O(1/t)$ convergence rate in the neighborhood of such equilibria.

An important point to make is that not every locally asymptotically stable equilibrium of $\dot x=-\omega(x)$ is a non-degenerate differential Nash equilibrium. 
Indeed, the following proposition provides  an entire class of games whose
corresponding gradient dynamics admit locally asymptotically stable equilibria that are not local Nash equilibria. 
\begin{proposition}
    In the class of general-sum continuous games, there exists a continuum of games containing games $\mc{G}$ such that $\LASE(\omega)\not\subset \NDDNE(\mc{G})$, and moreover, $\LASE(\omega)\not\subset\LNE(\mc{G})$.
    \label{prop:gsg}
\end{proposition}
\begin{proof}
    Consider a two player game $\mc{G}=(f_1,f_2)$ on $\mb{R}^2$ where
    \begin{align*}
        f_1(x_1,x_2)= \frac{a}{2}x^2_1 + bx_1x_2,\ \ \text{and}\ \   f_2(x_1,x_2)= \frac{d}{2}x^2_2 + cx_1x_2
    \end{align*}
    for constants $a,b,c,d \in \mb{R}$. The Jacobian of $\omega$ is given by
\begin{align}
\label{eq:domeg}
D\omega(x_1,x_2)=\bmat{a &   b \\ c & d}, \ \ \forall (x_1,x_2)\in \mb{R}^2.
\end{align}
If $a>0$ and $d<0$, then the unique stationary point $x=(0,0)$ is neither a differential Nash nor a local Nash equilibria since the necessary conditions are violated (i.e., $d<0$). However, if $a>-d$ and $ad>cb$, the eigenvalues of $D\omega$ have positive real parts and $(0,0)$ is asymptotically stable. Further, this clearly holds for a continuum of games. Thus, the set of locally asymptotically stable equilibria that are {not Nash equilibria} may be arbitrarily large. 
\end{proof}

The, preceding proposition shows that there exists attracting critical points of the gradient dynamics in general-sum continuous games that are not Nash equilibria and may not be even relevant to the game. Thus, this provides a negative answer to \textbf{Q2} (whether all attracting equilibria in general-games are game-relevant for the learning dynamics).

\begin{remark}
We note that, by definition, the non-Nash locally asymptotically stable equilibria (or non-Nash equilibria) do not satisfy the second-order conditions for Nash equilibria. Thus, at these joint strategies, at least one player -- and maybe all of them -- has a direction in which they would unilaterally deviate if they were not using gradient descent. As such, we view convergence to these points to be undesirable.
\end{remark}

\subsection{Zero-sum games}
Let us now restrict our attention to two-player zero-sum games, which often 
arise when training GANs, in adversarial learning,
and in MARL
\cite{goodfellow:2014aa,omidshafiei:2017aa,chivukula:2017aa}. In such games, one player can be seen as minimizing $f$ with respect to their decision variable and the other as minimizing $-f$ with respect to theirs. The following
proposition shows that all differential Nash equilibria in two-player zero-sum games are locally asymptotically stable equilibria under the flow of $\dot x=-\omega(x)$.

\begin{proposition}
\label{prop:zsg}
    For an arbitrary two-player zero-sum game, $(f,-f)$ on $\mb{R}^m$, if $x$ is
    a differential Nash equilibrium, then $x$ is both a non-degenerate differential Nash equilibrium and a locally asymptotically stable equilibrium of $\dot x=-\omega(x)$---that is, $\DNE(\mc{G})\equiv \NDDNE(\mc{G})\subset \LASE(\omega)$.
\end{proposition}

This result  guarantees that the differential Nash equilibria of zero-sum games are isolated and exponentially attracting under the flow of $\dot x=-\omega(x)$. This in turn guarantees that simultaneous gradient-play has a local linear rate of convergence to all local Nash equilibria in all zero-sum continuous games. Thus, the answer to \textbf{Q1} is the context of zero-sum games is ``yes'', since all Nash equilibria are attracting for the gradient dynamics.

The converse of the preceding proposition, however, is not
true. Not every locally asymptotically stable equilibrium in two-player zero-sum games are non-degenerate differential Nash equilibria. Indeed, there may
be many locally asymptotically stable equilibria in a zero-sum game that are not local Nash equilibria. 
The following proposition highlights this fact. 
\begin{proposition}
In the class of zero-sum continuous games, there exists a continuum of games  such that for each game $\mc{G}$, $\LASE(\omega)\not\subset \DNE(\mc{G})\subset \LNE(\mc{G})$.
\label{prop:zsg2}
\end{proposition}
\begin{proof}
    Consider the two-player zero-sum game $(\f, -\f)$ on $\mb{R}^2$ where
    \begin{align*}
        f(x_1,x_2)= \frac{a}{2}x^2_1 + bx_1x_2+\frac{c}{2}x^2_2;
    \end{align*}
    and $a,b,c \in \mb{R}$. The Jacobian of $\omega$ is given by
\[D\omega(x_1,x_2)=\bmat{\ \ a &\ \  b \\ -b & -c}, \ \ \forall \ (x_1,x_2)\in
\mb{R}^2.\]
If $a>c>0$ and $b^2>ac$, then $D\omega(x_1,x_2)$ has eigenvalues with strictly positive real part, but the unique stationary point is not a differential Nash equilibrium---since $-c<0$---and, in fact, is not even a Nash equilibrium. Indeed,
\begin{equation*} -f(0,0)>-f(0,x_2)=-\frac{c}{2}x_2^2, \ \ \forall\ x_2\neq 0.\end{equation*}
Thus, there exists a continuum of zero-sum games with a large set of locally asymptotically stable equilibria of the corresponding dynamics $\dot{x}=-\omega(x)$ that are not differential Nash.
\end{proof}

The, preceding proposition again shows that there exists non-Nash equilibria of the gradient dynamics in zero-sum continuous games. Thus, this proposition also provides a negative answer to \textbf{Q2} in the context of zero-sum games.

\subsection{Potential Games}
One last set of games with interesting connections between the Nash equilibria and the critical points of the gradient dynamics is the class known as \emph{potential games}. This particularly nice class of games are ones
for which $\omega$ corresponds to a gradient flow under a coordinate
transformation---that is, there exists a function $\phi$ (commonly referred to as the potential function) such that for each
$i\in\mc{I}$, $D_if_i\equiv
D_i\phi$. We remark that due to the equivalence this class of games is sometimes referred to as an \emph{exact} potential game. Note that a necessary and sufficient condition for
$(f_1,\ldots, f_\nplayers)$ to be a potential game is that $D\omega$ is
\emph{symmetric}~\cite{monderer:1996aa}---that is, $D_{ij}f_j\equiv
D_{ji}f_i$. This gives potential games the desirable property that the only locally asymptotically stable equilibria of the gradient dynamics  are local Nash equilibria.  

\begin{proposition}
    \label{prop:potentialgame}
    For an arbitrary potential game, $\mc{G}=(f_1,\ldots,f_\nplayers)$ on $\mb{R}^m$, if $x$ is
    a locally asymptotically stable equilibrium of $\dot x=-\omega(x)$ (i.e., $x\in\LASE(\omega)$), then $x$ is a non-degenerate differential Nash equilibrium (i.e., $x\in \NDDNE(\mc{G})$).
\end{proposition}

The full proof of Proposition~\ref{prop:potentialgame} is supplied in Appendix~\ref{app:proofs}. The preceding proposition rules out non-Nash locally asymptotically stable equilibria of the gradient dynamics in potential games, and implies that every local minimum of a potential game must be a local Nash equilibrium. Thus, in potential games, unlike in general-sum and zero-sum games, the answer to \textbf{Q2} is positive. However, the following proposition shows that the existence of a potential function is not enough to rule out local Nash equilibria that are saddle points of the dynamics.
\begin{proposition}
    In the class of continuous games, there exist a continuum of  potential games containing games $\mc{G}$ that admit Nash equilibria that are saddle points of the dynamics $\dot{x}=-\omega(x)$---i.e., $\exists\ \mc{G}$ such that for some $x\in \LNE(\mc{G})$, $x\in \SSP(\omega)$.
    \label{ex:potssp}
\end{proposition}
\begin{proof}
    Consider the game $(\f, \f)$ on $X=\mb{R}^2$ described by
    \[ f(x_1,x_2)=\frac{a}{2}x_1^2 + bx_1x_2+\frac{c}{2}x_2^2\]
    where $a,b,d \in \mathbb{R}$. The Jacobian of $\omega$ is given by
    \[D\omega(x_1,x_2)=\bmat{a &  b \\ b & c}, \ \ \forall \ (x_1,x_2)\in
\mb{R}^2.\]
    If $a,c>0$, then $x=(0,0)$ is a local Nash equilibrium. However, if $ac<b^2$,
    $D\omega(x)$ has one positive  and one negative eigenvalue and $(0,0)$ is a saddle point of the gradient dynamics. Thus, there exists a continuum of
potential games where a large set of differential Nash equilibria are strict saddle points of 
$\dot{x}=-\omega(x)$.
\end{proof}

Proposition~\ref{ex:potssp} demonstrates a surprising fact about potential games. Even though all minimizers of the potential function must be local Nash equilibria, \emph{not all local Nash equilibria are minimizers of the potential function}. 

\subsection{Main Takeaways}

The main takeaways of this section are summarized in Figure~\ref{fig:eqtype}. We note that for zero-sum games, Proposition~\ref{prop:zsg2} shows that $\LNE(\mc{G}) \subset \LASE(\omega)$. Since the inclusion is strict, the answer to \textbf{Q2} in such games is ``no''. For general-sum games, Proposition~\ref{prop:gsg} allows us to to conclude that there do exist attracting, non-Nash equilibria. Thus, the answer to  \textbf{Q2} is also ``no''.   In potential games, since $\LASE(\omega) \subset \LNE(\mc{G})$ the answer is ``yes''.

In the following sections, we provide answers to \textbf{Q1} by showing that all local Nash equilibria in $\LNE(\mc{G}) \cap \SSP(\omega)$ are avoided almost surely by gradient-based algorithms in both the deterministic and stochastic settings. In particular, since $\LNE(\mc{G}) \cap \SSP(\omega) \ne \emptyset$ in potential and general-sum games, one cannot give a positive answer to \textbf{Q1} in either of these classes of games.

\section{Convergence of Gradient-Based Learning}
\label{sec:results}
In this section, we provide convergence and non-convergence results for gradient-based algorithms. We also include a high-level overview of well-known
algorithms that fit into the class of learning algorithms we consider; more
detail can be found in Appendix~\ref{sec:examples}.
\subsection{Deterministic Setting}
\label{sec:fullinfo}
We first address convergence to equilibria in the \emph{deterministic} setting in which agents have oracle access to their gradients at each time step. This includes 
the case where agents know their own cost functions $f_i$ and observe their own
actions as well as their competitors' actions---and hence, can compute the
gradient of their cost with respect to their own choice variable.

Since we have assumed that each agent $i\in \mc{I}$ has their own \emph{learning rate} (i.e.~step sizes
$\gamma_i$), the joint dynamics of all the players are given by
\begin{equation}
    \pxt{t+1}=g(\pxt{t})
\end{equation}
where $g:x\mapsto x-\gamma \odot \omega(x)$ with
$\gamma=(\gamma_i)_{i\in \mc{I}}$ and $\gamma>0$
element-wise. By a slight abuse of notation, 
$\gamma\odot \omega(\pxt{t})$ is defined to be element-wise multiplication of $\gamma$ and
$\omega(\cdot)$ where $\gamma_1$ is multiplied by the first $m_1$ components of
$\omega(\cdot)$, $\gamma_2$ is multiplied by the next $m_2$ components, and so on. 

We remark that this update rule immediately distinguishes gradient-based learning in games from gradient descent. By definition, the dynamics of gradient descent in single-agent settings always correspond to gradient flows ---i.e $x$ evolves according to an ordinary differential equation of the form $\dot x=-\nabla \phi(x)$ for some function $\phi:\mb{R}^d \rar \mb{R}$. Outside of the class of \emph{exact} potential games we defined in Section~\ref{sec:connections}, the dynamics of players' actions in games are not afforded this luxury---indeed, $D\omega$ is not in general symmetric (which is a necessary condition for a gradient flow). This makes the potential limiting behaviors of $\dot{x}=-\omega(x)$ highly non-trivial to characterize in general-sum games. 

The structure present in a gradient-flow implies strong properties on the limiting behaviors of $x$. In particular, it precludes the existence of limit cycles or periodic orbits (limiting behaviors of dynamical systems where the state of system cycles infinitely through a set of states with a finite period) and chaos (an attribute of nonlinear dynamical systems where the system's behavior can vary extremely due to slight changes in initial position) \cite{sastry:1999aa}.   We note that both of these behaviors can occur in the dynamics of gradient-based learning algorithms in games\footnote{The Van der Pol oscillator and Lorenz system  (see e.g \cite{sastry:1999aa}) can be seen as the resulting gradient dynamics in a 2-player and 3-player general-sum game respectively. The first is a classic example of a system where players converge to cycles and the second is an example of a chaotic system.}. 

Despite the wide breadth of behaviors that gradient dynamics can exhibit in competitive settings, we are still make statements about convergence (and non-convergence) to certain types of equilibria. To  do so,  we first make the following standard assumptions on the smoothness of the cost functions $f_i$ and the magnitude of the agents' learning rates $\gamma_i$.
\begin{assumption}
    For each $i\in \mc{I}$, $\f_i\in C^\dimsmth({X}, \mb{R})$ with $\dimsmth\geq 2$,
    $\sup_{x\in X}\|D\omega(x)\|_2\leq L<\infty$, and 
    $0<\gamma_i<1/L$ where $\|\cdot\|_2$ is the induced $2$-norm.
    \label{ass:ell}
\end{assumption}

Given these assumptions, the following result rules out converging to strict saddle points.

\begin{theorem}
    Let $\f_i:{X}\rar \mb{R}$ and $\gamma$ satisfy Assumption~\ref{ass:ell}.
    Suppose that
 ${X}=X_1\times \cdots \times X_\nplayers \subseteq\mb{R}^m$ is open and convex.
 If
    $g({X})\subset {X}$, the  set of initial conditions
    $x\in {X}$ from which competitive gradient-based learning converges to 
   strict saddle points is of measure zero.
    \label{thm:fullinfo}
\end{theorem}

We remark that the above theorem holds for $X=X_1\times \cdots \times
X_\nplayers=\mb{R}^m$ in particular, since $g(X)\subset X$ holds trivially in this case.
It is also important to note that, as we point out in
Section~\ref{sec:connections}, local Nash equilibria can be strict saddle
points. Thus, all local Nash equilibria that are strict saddle points for $\dot{x}=-\omega(x)$ are avoided almost surely by gradient-play even with oracle gradient access and random initializations. This holds even when players randomly initialize uniformly in an arbitrarily small ball around such Nash equilibria. In Section~\ref{sec:LQR}, we show that many linear quadratic dynamic games have a strict saddle point as their global Nash equilibrium. For brevity, we provide the proof of Theorem~\ref{thm:fullinfo} in Appendix~\ref{app:proofs}, and provide a proof sketch below.

\begin{proof}[Proof sketch of Theorem~\ref{thm:fullinfo}]
The core of the proof is the celebrated stable manifold theorem from dynamical systems
theory, presented in Theorem~\ref{thm:centerstable}. We construct the set of initial positions from which gradient-play will converge to strict saddle points and then use the stable manifold theorem to show that the set must have measure zero in the players' joint strategy space. Therefore, with a random initialization players will never evolve solely on the stable manifold of strict saddles and they will consequently diverge from such equilibria. 

To be able to invoke the stable manifold theorem, we first show that
the mapping $g: \mb{R}^m\rar \mb{R}^m$ is a diffeomorphism, which is non-trivial due to the fact that we have allowed each agent to have their own learning rate 
$\gamma_i$ and $D\omega$ is not symmetric. We then iteratively construct the set
of initializations that will converge to strict saddle points under the game
dynamics. By the stable manifold theorem, and the fact that $g$ is a
diffeomorphism, the stable manifold of a strict saddle point must be measure
zero. Then, by induction we show that the set of all initial points that converge to a strict saddle point must also be measure zero.
\end{proof}

In potential games we can strengthen the above non-convergence result and give convergence guarantees. 
\begin{corollary}
    Consider a potential game $(f_1,\ldots, f_\nplayers)$ on open, convex $X=X_1\times \cdots \times
    X_\nplayers\subseteq \mb{R}^m$ and where
 each $f_i\in C^\dimsmth(X, \mb{R})$ for $\dimsmth\geq 3$. 
    Let $\nu$ be a prior measure with support $X$ which is
    absolutely continuous with respect to the Lebesgue measure and assume
    $\lim_{t\rar \infty} g^t(x)$ exists. Then,
    under Assumption~\ref{ass:ell},  competitive gradient-based learning converges 
     to 
     non-degenerate differential
    Nash equilibria almost surely. Moreover, the non-degenerate differential
    Nash to which it converges is generically a local Nash equilibrium.
    \label{cor:msfinite}
\end{corollary}

Corollary~\ref{cor:msfinite} guarantees that in potential games, gradient-play will converge to a differential Nash equilibrium. Combining this with Theorem~\ref{sec:fullinfo} guarantees that the differential Nash equilibrium it converges to is a local minimizer of the potential function. A simple implication of this result is that gradient-based learning in potential games cannot exhibit limit cycles or chaos. 

Of note is the fact that the agents \emph{do not} need to be performing gradient-based learning on $\phi$ to converge to Nash almost surely. That is, they do not need to know the function $\phi$; they simply need to follow the derivative of their own cost with respect to their own choice variable, and they are guaranteed to converge to a local Nash equilibrium that is a local minimizer of the potential function. 

We note that convergence to Nash equilibria is a known characteristic of gradient-play in potential games. However, our analysis also highlights that gradient-play will avoid a subset of the Nash equilibria of the game. This is surprising given the particularly strong structural properties of such games. The proof for Corollary~\ref{cor:msfinite} is provided in Appendix~\ref{app:proofs} and follows from Proposition~\ref{prop:potentialgame}, Theorem~\ref{thm:fullinfo}, and the fact that $D\omega$ is symmetric in potential games.

\subsubsection{Implications and Interpretation of Convergence Analysis}
\label{sec:implications}
Both Theorem~\ref{thm:fullinfo} and Corollary~\ref{cor:msfinite} show that gradient-play in multi-agent settings avoids strict saddles almost surely even in the deterministic setting. Combined with the analysis in Section~\ref{sec:connections} which shows that (local) Nash equilibria can be strict saddles of the dynamics for  general-sum games, this implies that a subset of the Nash equilibria are almost surely avoided by individual gradient-play, a potentially undesirable outcome in view of \textbf{Q1} (whether all Nash equilibria are attracting for the learning dynamics).
In Section~\ref{sec:LQR}, we show that the global Nash equilibrium is a saddle point of the gradient dynamics in a large number of randomly sampled LQ dynamic games. This suggests that policy gradient algorithms may fail to converge in such games, which is highly undesired. This is in stark contrast to the single agent setting where policy gradient has been shown to converge to the unique solution of LQR problems \cite{Fazel2018GlobalCO}.  

In Section~\ref{sec:connections}, we also showed that local Nash equilibria of potential games can be strict saddles points of the potential function. Non-convergence to such points in potential games is not necessarily a bad result since this in turn implies convergence to a local minimizer of the potential function (as shown in~\cite{lee:2016aa, panageas:2016aa}) which are guaranteed to be local Nash equilibria of the game. However, these results do imply that \emph{one cannot answer ``yes'' to \textbf{Q1} in potential games} since some of the Nash equilibria are not attracting under gradient-play.

In zero-sum games, where local Nash equilibria cannot be strict saddle points of the gradient dynamics, our result  suggests that \emph{eventually} gradient-based learning algorithms will escape saddle points of the dynamics. 

The almost sure avoidance of all equilibria that are saddle points of the dynamics further implies that if \eqref{eq:sys} converges to a critical point $x$, then $x\in \LASE(\omega)$---i.e., $x$ is locally asymptotically stable for $\dot{x}=-\omega(x)$. This may not be a desired property however, since we showed in Section~\ref{sec:connections} that zero-sum and general-sum games both admit non-Nash LASE.

Since  gradient-play in games generally does not result in a gradient flow, other types of limiting behaviors such as limit cycles can occur in gradient-based learning dynamics. Theorem~\ref{thm:fullinfo} says nothing about convergence to other
limiting behaviors. In the following sections we prove that the results described in this section extend to the stochastic gradient setting. We also formally define periodic orbits in the context of dynamical systems and state stronger results on avoidance of some more complex limiting behaviors like linearly unstable limit cycles.

\subsection{Stochastic Setting}
\label{sec:gradientfree}

We now analyze the stochastic case in which agents are assumed to have an
unbiased estimator for their gradient. 
The results in this section allow us to extend the
results from the deterministic setting to a setting where each agent builds
an estimate of the gradient of their loss at the current set of strategies from
potentially noisy observations of the environment. Thus, 
we are able to analyze the limiting behavior of a class of commonly used machine
learning algorithms for competitive, multi-agent settings. 
In particular, we show 
that agents will almost surely not converge to strict saddle points. In Appendix~\ref{app:repel}, we show that the gradient dynamics will actually avoid more general limiting behaviors called linearly unstable cycles which we define formally. 

To perform our analysis, we make use of tools and ideas from the literature on stochastic approximations (see e.g \cite{BorkarStoch}). We note that the convergence of stochastic gradient schemes in the single-agent setting has been extensively studied \cite{robbin:1971aa,pemantle:1990aa,BottouSGD,MertReviewer}. We extend this analysis to the behavior of stochastic gradient algorithms in games.

We assume that each agent updates their strategy using the update rule
\begin{equation}
\pxtwo{i}{t+1}=x_{i,t}-\gamma_{i,t}(D_if_i(x_{i,t},x_{-i,t})+w_{i,t+1}) %
 \label{eq:sa}
\end{equation}
for some zero-mean, finite-variance stochastic process $\{w_{i,t}\}$.
Before presenting the results for the stochastic case, let us comment on the different learning algorithms that fit into this framework.

\subsubsection{Examples of Stochastic Gradient-Based Learning}
\label{subsec:gradientalgs}
\bgroup
\def\arraystretch{0.9}
{\setlength{\tabcolsep}{0.2em}    \begin{table}[t]
    \centering
    \begin{tabular}{|c||c|c|}
        \hline\textbf{Class}  &
        \textbf{Gradient Learning Rule} \\ %
        \hline\hline
         \multirow{2}{*}{Gradient-Play}  &
        \multirow{2}{*}{$x_{i}^+=x_{i}-\gamma_iD_if_i(x_{i}, x_{-i})$}\\ &  \\\hline
        \multirow{2}{*}{GANs}  &
        $\theta^{+}\ =\theta-\gamma \mb{E}[D_{\theta}L(\theta,w)]\ $\\
        &
        $w^+=w+\gamma \mb{E}[D_{w}L(\theta,w)]$ \\\hline
        \multirow{2}{*}{ MA Policy Gradient}
        &\multirow{2}{*}{$x_{i}^+=x_{i}-\gamma_i\mb{E}[{D_iJ_i}(x_{i},
        x_{-i})]$}\\  & \\ \hline    
        \multirow{2}{*}{Individual Q-learning} &
        \multirow{2}{*}{$q_i^+(u_i)=q_i(u_i)+\gamma_i(r_i(u_i,
            \pi_{-i}(q_i,q_{-i}))-q_i(u_i))$}
            \\  & \\\hline
        \multirow{1}{*}{MA Gradient Bandits}
        &$x_{i,\ell}^+=x_{i,\ell}+\gamma_i\mb{E}[\beta_iR_i(u_i,u_{-i})|u_i=\ell]$, $\ell=1,\ldots, m_i$\\ \hline
        \multirow{1}{*}{MA Experts} &  $x_{i,\ell}^+=x_{i,\ell}+\gamma_i\mb{E}[R_i(u_i,u_{-i})|u_i=\ell]$,  $\ell=1,\ldots, m_i$\\\hline
    \end{tabular}
    \caption{Example problem classes that fit into competitive gradient-based learning
    rules. Details on the derivation of these update rules as gradient-based
    learning schemes is provided in Appendix~\ref{sec:examples}. }
    \label{tab:examples}
\end{table}}

The stochastic gradient-based learning setting we study is general enough to
include a variety of commonly used multi-agent learning algorithms. The classes of algorithms we include is hardly an exhaustive list, and indeed many extensions and altogether different algorithms exist that can be considered members of this class. 
In Table~\ref{tab:examples}, we provide the gradient-based update rule for six
different example classes of learning problems: (i) gradient-play in
non-cooperative continuous games, (ii) GANs, (iii) multi-agent policy gradient, (iv) individual
Q-learning, (v) multi-agent gradient bandits, and (vi) multi-agent experts. We provide a detailed analysis of these different algorithms
including the derivation of the gradient-based update rules along with some
interesting numerical examples in
Appendix~\ref{sec:examples}. In each of these cases, one can view an agent employing the given algorithm as building an unbiased estimate of their gradient from their observation of the environment.

For example, in multi-agent policy gradient (see, e.g.,~\cite[Chapter~13]{sutton:2017aa}),
 agents' costs are defined as functions of a parameter vector $x_i$ that
parameterize their policies $\pi_{i}(x_i)$.  The parameters $x_i$ are agent
$i$'s choice variable. By following the gradient of their loss function, they aim to tune the
parameters in order to converge to an \emph{optimal} policy $\pi_i$. Perhaps surprisingly, it is not necessary for agent
$i$ to have access to $\pi_{{-i}}(x_{-i})$ or even $x_{-i}$ in order for them to construct an
unbiased estimate of the gradient of their loss with respect to their own
choice variable $x_i$ as long as they observe the sequence of actions, say
$u_{-i,t}$, of
all other agents generated. These actions are implicitly determined by the other
agents' policies $\pi_{-i}(x_{-i})(\cdot)$. Hence, in this case if agent $i$
observes
$\{(r_{j,t},u_{j,t},s_{j,t})$, $\forall \ j\in
\mc{I}\}$
where $(r_j,u_j,s_j)$ are the reward, action, and state of agent $j$, then this
is enough to construct an unbiased estimate of their gradient. 
We provide further details on multi-agent policy gradient in
Appendix~\ref{sec:examples}.

\subsubsection{Stochastic Gradient Results}
Returning to the analysis of \eqref{eq:sa},
we make the following standard assumptions on the noise processes \cite{robbin:1971aa,robbins:1985aa}.
\begin{assumption}
    The stochastic process $\{w_{i,t+1}\}$ satisfies the assumptions
$\mb{E}[w_{i,t+1}|\ \mc{F}_i^t]=0$,  $t\geq 0$
and 
$\mb{E}[\|w_{i,t+1}\|^2|\ \mc{F}_i^t]\leq \sigma^2<\infty$ a.s., 
for $t\geq 0$, where $\mc{F}_{i,t}$ is an increasing family of
$\sigma_i$-fields---i.e.~filtration, or history generated by the sequence of random
variables---given by
$\mc{F}_{i,t}=\sigma_i(\pxtwo{i}{k},w_{i,k}, k\leq t), \
t\geq 0$.
\label{ass:estmartin}
\end{assumption}
We also make new assumptions on the players' step-sizes. These are standard assumptions in the stochastic approximation literature and are needed to ensure that the noise processes are asymptotically controlled.
\begin{assumption}
    For each $i\in\mc{I}$, $f_i\in C^\dimsmth(X,\mb{R})$ with $\dimsmth\geq 2$, $D_{i}\f_i$ is
    $L_i$--Lipschitz with 
    $0<L_i<\infty$, the step-sizes satisfy $\gamma_{i,t}\equiv \gamma_t$ for all
    $i\in \mc{I}$ and 
$\sum_t \gamma_t=\infty$ and $\sum_t (\gamma_t)^2<\infty$, and 
$\sup_t\|\pxt{t}\|<\infty$ a.s. 
\label{ass:others}
\end{assumption}
Let $(a)^+=\max\{a, 0\}$ and $a\cdot b$ denotes the inner product. The following theorem extends the results of Theorem~\ref{thm:fullinfo} to the stochastic gradient dynamics in games.
\begin{theorem}
    Consider a game $(f_1,\ldots, f_\nplayers)$ on $X=X_1\times \cdots \times
    X_n=\mb{R}^m$. Suppose each agent $i\in \mc{I}$ adopts a stochastic gradient algorithm that satisfies Assumptions~\ref{ass:estmartin}
   and \ref{ass:others}. Further, suppose that for each $i\in \mc{I}$, there
   exists a constant $b_i>0$
   such that
   $\mb{E}[(w_{i,t}\cdot v)^+|\mc{F}_{i,t}]\geq b_i$
   for every unit vector $v\in \mb{R}^{m_i}$. Then,
    competitive stochastic gradient-based learning converges to strict saddle
   points of the game
   on a set of measure
   zero. 
    \label{thm:gradfree}
\end{theorem}
The proof follows directly from showing that \eqref{eq:sa}
satisfies Theorem~\ref{thm:pementle}, provided the assumptions of the theorem hold. The assumption that $\mb{E}[(w_{i,t}\cdot v)^+|\mc{F}_{i,t}]\geq b_i$ rules out degenerate cases where the noise forces the stochastic dynamics onto the stable manifold of strict saddle points.

Theorem~\ref{thm:gradfree} implies that the dynamics of stochastic gradient-based learning defined in \eqref{eq:sa}, have the same limiting properties as the deterministic dynamics vis-\`a-vis saddle points. Thus, the implications described in Section~\ref{sec:implications} extend to the stochastic gradient setting. In particular, stochastic gradient-based algorithms will avoid a non-negligible subset of the Nash equilibria in general-sum and potential games. Further, in zero-sum and general-sum games, if the players fo converge to a critical point, that point may be a non-Nash equilibrium.

\subsubsection{Further Convergence Results for Stochastic Gradient-Play in Games} 

As we demonstrated in Section~\ref{sec:fullinfo}, outside of potential games, the dynamics of gradient-based learning algorithms in games are not gradient flows. As such, the players' actions can converge to more complex sets than simple equilibria. A particularly prominent class of limiting behaviors for dynamical systems are known as limit cycles (see e.g \cite{sastry:1999aa}). Limit cycles (or periodic orbits) are sets of states $\mc{S}$ such that each state $x \in \mc{S}$ is visited at periodic intervals \emph{ad infinitum} under the dynamics. Thus, if the gradient-based algorithms converge to a limit cycle they will cycle infinitely through the same sequence of actions. Like equilibria, limit cycles can be stable or unstable under the dynamics $\dot x=-\omega(x)$, meaning that the dynamics can either converge to or diverge from them depending on their initializations. 

We remark that the existence of oscillatory behaviors and limit cycles has been observed in the dynamics of of gradient-based learning in various settings like the training of Generative Adversarial Networks \cite{daskalakis:2017aa}, and multiplicative weights in finite action games \cite{mertikopoulos:2018aa}. We simply emphasize that the existence of such limiting behaviors is due to the fact that the dynamics are no longer gradient flows. This fact also allows for other complex limiting behaviors like chaos\footnote{A general term used to characterize dynamical systems where arbitrarily small perturbations in the initial conditions lead to drastically different solutions to the differential equations} to exist in the dynamics of gradient-based learning in games. We also show in Appendix~\ref{app:repel} that gradient-based learning  avoids some limit cycles.

In Appendix~\ref{app:repel}, we formalize the notion of a limit cycle and its stability in the stochastic setting. Using these concepts, we then provide an analogous theorem to Theorem~\ref{thm:gradfree} which states that competitive stochastic gradient-based learning converges to linearly unstable limit cycles---a parallel notion to strict saddle points but pertaining to more general limit sets---on a set of measure zero, provided that analogous assumptions to those in the statement of Theorem~\ref{thm:gradfree} hold. Providing such guarantees requires a bit more mathematical formalism, and as such we leave the details of these results to Appendix~\ref{app:ms}.

 In pursuit of a more general class of games with  desirable convergence properties, in Appendix~\ref{app:msms} we also introduce a generalization of potential games, namely Morse-Smale games, for which the combined gradient dynamics correspond to a Morse-Smale vector field~\cite{hirsch:1976aa,palis:1970aa}. In such games players are guaranteed to converge to only (linearly stable) cycles or equilibria. In such games, however, players may still converge to non-Nash equilibria and avoid a subset of the Nash equilibria.

\section{Saddle Point LNE in LQ Dynamic Games}
\label{sec:LQR}

In this section, we present empirical results that show that a non-negligible subset of
two-player LQ games have local Nash equilibria that are strict saddle points
of the gradient dynamics. LQ games serve as good benchmarks for analyzing the
limiting behavior of gradient-play in a non-trivial setting since they are known to admit global Nash equilibria that can be found be solving a coupled set of Riccati equations \cite{BasarOlsder}. LQ games can also be cast as multi-agent reinforcement learning problems where each agent has a policy that is a linear function of the state and a quadratic reward function. Gradient-play in LQ games can therefore be seen as a form of policy gradient.

The empirical results we now present imply that, even in the relatively straightforward case of linear dynamics, linear feedback policies, and quadratic costs, policy gradient
multi-agent reinforcement learning would be unable to find the local Nash equilibrium in a non-negligible subset of problems.

\paragraph{LQ game setup} For simplicity, we consider two-player LQ games in $\mathbb{R}^2$. Consider a discrete time dynamical
system defined by
\begin{align}
z(t+1)=Az(t)+B_1u_1(t)+B_2u_2(t)
\label{eq:updateLQR}
\end{align}
where $z(t) \in \mathbb{R}^2$ is the state at time $t$, $u_1(t)$ and $u_2(t)$ are
the control inputs of players $1$ and $2$, respectively, and $A$, $B_1$, and $B_2$ are the system matrices. We assume that player
$i$ searches for a linear feedback policy of the form $u_i(t)=-K_iz(t)$
that minimizes their loss which is given by
\[\textstyle f_i(z_0,u_1,u_2)=\sum_{t=0}^\infty z(t)^TQ_iz(t)+u_{i}(t)^TR_iu_{i}(t)\] 
where $Q_i \succ 0$ and $R_i \succ 0$ are the cost matrices on the state and
input, respectively. We note that the two players are coupled
through the dynamics since $z(t)$ is constrained to obey the update equation
\eqref{eq:updateLQR}. 
The vector of player derivatives is given by $\omega(K_1,K_2)=(D_1f_1(K_1,K_2),D_2f_2(K_1,K_2))$  where
\[D_if_i(K_1,K_2)
\textstyle=(R_{ii}K_i+B_i^TP_i(B_1K_1+B_2K_2)-B_i^TP_iA)\sum_{t=0}^\infty
z(t)z(t)^T, \ i\in\{1,2\}.\]
Note that there is a slight abuse of notation here as we are treating $D_if_i$
as a matrix and as the vectorization of a matrix.  %
 The matrices $P_1$ and $P_2$ can be found by solving the Riccati equations 
 \begin{align*}
P_i & = (A-B_1K_1-B_2K_2)^TP_i(A-B_1K_1-B_2K_2)+ K_i^TR_iK_i +Q_i, \ \
i\in\{1,2\},
\end{align*}
 for a given
$(K_1,K_2)$.
As shown in \cite{BasarOlsder}, global Nash equilibria of LQ games can be found by solving coupled Ricatti equations. Under the following assumption, this can be done using an analogous method to the method of Lyapunov iterations outlined in \cite{LyapIterCitation} for continuous time LQ games. 

\begin{assumption}
Either $(A,B_1,\sqrt{Q_1})$ or $(A,B_2,\sqrt{Q_2})$ is stabilizable-detectable.
\end{assumption}

Further information on the uniqueness of Nash equilibria in LQ games and the method of Lyapunov iterations can be found in \cite{BasarOlsder} and \cite{LyapIterCitation} respectively.  

\paragraph{Generating LQ games with strict saddle point Nash equilibria} Without loss of generality, we assume $(A,B_1,\sqrt{Q_1})$ is stabilizable-detectable. Given that we have a method of finding the global Nash equilibrium of the LQ game, we now present our experimental setup. 

We fix $B_1$, $B_2$, $Q_1$, and $R_1$ and parametrize $Q_2$, and $R_2$ by  $q$ and $r$ respectively. The shared dynamics matrix $A$ has entries that are sampled from the uniform distribution supported on $(0,1)$.  For each value of the parameters $b$, $q$, and $r$, we randomly sample $1000$ different $A$ matrices. Then, for each LQ game defined in terms of each of the sets of parameters, we find the optimal feedback matrices $(K_1^*,K_2^*)$ using the method of Lyapunov iterations, and we numerically approximate $D\omega(K_1^*,K_2^*)$ using auto-differentiation tools and check its eigenvalues.

The exact values of the matrices are defined as follows: $A \in \mb{R}^{2\times 2}$ with each of the entries $a_{ij}$ sampled from the uniform distribution on $(0,1)$,
\begin{align*}
    B_1=\begin{bmatrix}1\\1\end{bmatrix}, \
    B_2=\begin{bmatrix}0\\1\end{bmatrix}, \ Q_1=\begin{bmatrix}0.01 & 0\\0 &
        1\end{bmatrix}, \ Q_2=\begin{bmatrix}1 & 0\\0 & q \end{bmatrix}, \
    R_1=0.01, \ R_2=r.
\end{align*}

The results for various combinations of the parameters  $q$ and $r$ are shown in Figure~\ref{fig:lqr}. For all of the different parameter configurations considered, we found that in anywhere from $0\%-25\%$ of the randomly sampled LQ games, there was a global Nash equilibrium that was a strict saddle point of the gradient dynamics. Of particular interest is the fact that for all values of $q$ and $r$ we tested, at least $5\%$ of the LQ games had a global Nash equilibrium with the strict saddle property. In the worst case, around $25\%$ of the LQ games for the given values of $q$ and $r$ admitted such Nash equilibria. 

 \begin{figure}[t]
\center    
      \includegraphics[width=0.75\textwidth]{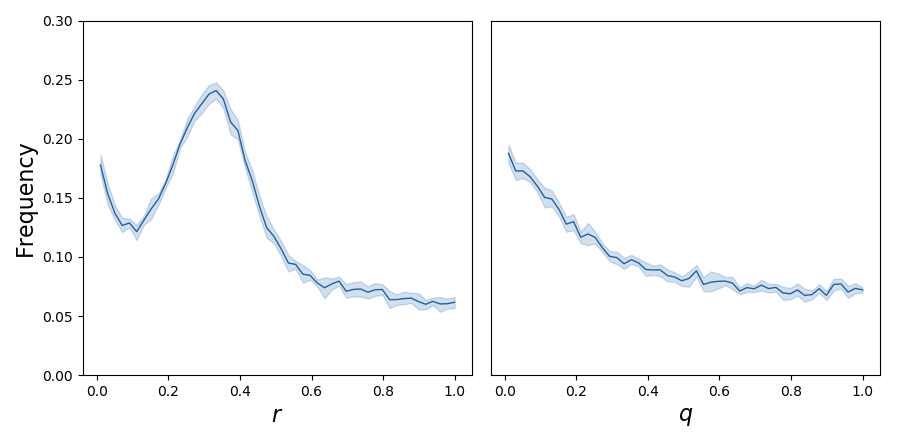}
        \caption{Frequency (out of 1000) of randomly sampled LQ games with global Nash equilibria that are avoided by policy-gradient. The experiment was run 10 times and the average frequency is shown by the solid line. The shaded region demarcates the $95\%$ confidence interval of the experiment.   (left) $r$ is varied in $(0,1)$, $q=0.01$. (right) $q$ is varied in $(0,1)$, $r=0.1$.}
        \label{fig:lqr}
\end{figure}

\begin{remark}
These empirical observations imply that multi-agent policy gradient, even in the relatively straightforward setting of linear dynamics, linear policies, and quadratic costs, has no guarantees of convergence to the global Nash equilibria in a non-negligible number of games. Further investigation is warranted to validate this fact theoretically. This in turn supports the idea that for more complicated cost functions, policy classes, and dynamics, local Nash equilibria with the strict saddle property are likely to be very common. 
\end{remark}

\section{Discussion and Future Directions}
\label{sec:discussion}

In this paper we provided answers to the following two questions for classes of gradient-based learning algorithms:

\begin{description}[leftmargin=20pt]
    \item[Q1.] \emph{Are all attractors of the learning algorithms employed by agents equilibria relevant to the underlying game?}
    \item[Q2.] \emph{Are all equilibria relevant to the game also attractors of the learning algorithms agents employ?}
\end{description}

We answered these questions in general-sum, zero-sum, and potential games without imposing structure on the game outside regularity conditions on the cost functions by exploiting the observation that gradient-based learning dynamics are not gradient flows. Our analysis, was shown in Section~\ref{sec:examples} to apply to a number of commonly used methods in multi-agent learning. 

\subsection{Links with Prior Work}
As we noted, previous work on learning in games in both the game theory literature, and more recently from the machine learning community,
has largely focused on \textbf{Q1}, though some recent work has analyzed \textbf{Q2} in the setting of zero-sum games.

 In the seminal work by Rosen \cite{Rosen1965},  $n$--player concave or monotone games are shown to either admit a unique Nash equilibrium or a continuum of Nash equilibria, all of which are attracting under gradient-play. The structure present in these games rules out the existence of non-Nash equilibria.

Two-player, finite-action bilinear games have also been extensively studied. In \cite{GradDyn}, the authors investigate the convergence of the gradient dynamics in  such games. Additionally, the dynamics of other (non gradient-based) algorithms like multiplicative weights have been studied in \cite{hommes:2012aa} among many others. In such settings, the structure guarantees that there exists a unique global Nash equilibrium and no other critical points of the gradient dynamics. As such, non-Nash equilibria, cannot exist.

In the study of learning dynamics in the class of zero-sum games, it has been shown that cycles can be attractors of the dynamics (see, e.g.,~\cite{mertikopoulos:2018aa, Wesson2016, hommes:2012aa}). Concurrently with our results, \cite{Daskalakis} also showed the existence of non-Nash attracting equilibria in this setting. 

In more general settings, there has been some analysis of the limiting behavior of gradient-play though the focus has been for the most part, on giving sufficient conditions under which Nash equilibria are attracting under gradient-play. For example, ~\cite{ratliff:2013aa,ratliff:2014aa,ratliff:2016aa}, introduced the notion of a differential Nash equilibrium which is characterized by first and second order conditions on the players' individual cost functions and which we made extensive use of. Following this body of work, \cite{Mertikopoulos2019} also investigated the local convergence of gradient-play in continuous games. They showed that if a Nash equilibrium satisfies a property known as \emph{variational stability}, the equilibrium is attracting under gradient play. In twice continuously differentiable games, this condition coincides exactly with the definition of stable differential Nash equilibria. Though these works  analyze a general class of games, the focus of the  analysis is solely on the local characterization and computation (via gradient play) of local Nash equilibria. As such, the issues of non-convergence that we show in this paper were not discussed.

\subsection{Open Questions}
Our results suggest that gradient-play in multi-agent settings has
fundamental problems. Depending on the players' costs, in general games and even potential games, which have a particularly \emph{nice} structure, a subset of the Nash equilibria will be almost surely avoided by gradient-based learning when the agents randomly initialize their first action. In zero-sum and general-sum games, even if the algorithms do converge, they may have converged to a point that has no game theoretic relevance, namely a non-Nash locally asymptotically stable equilibrium.

Lastly, these results show that limit cycles persist even under a stochastic update scheme. This
explains the empirical observations of limit cycles in gradient dynamics presented in ~\cite{daskalakis:2017aa,leslie:2005aa,hommes:2012aa}. It also implies
that gradient-based learning in multi-agent reinforcement learning, multi-armed bandits, generative
adversarial networks, and online optimization all admit limit cycles under certain loss functions. Our empirical results show that these problems are not merely of theoretical interest, but also have great relevance in practice. 

Which classes of games have all Nash being attracting for gradient-play and which classes preclude the existence of non-Nash equilibria is an open and particularly interesting question. Further, the question of whether gradient-based algorithms can be constructed for which only game-theoretically relevant equilibria are attracting is of particular importance as gradient-based learning is increasingly implemented in game theoretic settings. Indeed, more generally, as learning algorithms are increasingly deployed in markets and other competitive environments understanding and dealing with such theoretical issues will become increasingly important.

\appendix
\section{Proofs of the Main Results}
\label{app:proofs}
This appendix contains the full proofs of the results in the paper.
\subsection{Proofs on Links Between Dynamical Systems and Games}
We begin with a proof of Proposition~\ref{lem:nddne} that all differential Nash equilibria are either strict saddle points or asymptotically stable equilibria of the gradient dynamics. This relies mainly on the definitions of strict saddle points, locally asymptotically stable equilibria, and non-degenerate differential Nash equilibria and simple linear algebra.

\begin{proof}[Proof of Proposition~\ref{lem:nddne}]
    Suppose that $x\in X$ is a non-degenerate differential Nash equilibrium.
We claim that $\tr(D\omega(x))>0$. 
Since
 $x$ is a differential Nash equilibrium,  $D_{i}^2f_i(x)\succ0$ for each $i
 \in \mc{I}$; these are the diagonal blocks of $D\omega(x)$. Further
 $D_{i}^2f_i(x)\succ0$ implies that  $\tr(D_{i}^2f_i(x))>0$. Since
 $\tr(D\omega)=\sum_{i=1}^n\tr(D_{i}^2f_i(x))$, $\tr(D\omega(x))>0$. Thus,
 it is not possible for all the eigenvalues to have negative real part. Since
 $x$ is non-degenerate,
$\det(D\omega(x))\neq 0$ so that
 none of the eigenvalues can have zero real part. Hence, at least one eigenvalue has strictly positive real part.

 To complete the proof, we show that the conditions for non-degenerate differential Nash equilibrium are not sufficient to
 guarantee that $x$ is locally asymptotically stable for the gradient dynamics---that is, not all eigenvalues of $D\omega(x)$
 have strictly positive real part. We do this by constructing a class of games
 with the strict saddle point property.
 Consider a class of two player games $\mc{G}=(f_1,f_2)$ on $\mb{R}\times \mb{R}$ defined as follows:
\[ (f_1(x_1,x_2),f_2(x_1,x_2))=\Big(\frac{a}{2}x_1^2+bx_1x_2, \frac{d}{2}x_2^2+cx_1x_2\Big). \]
In this game, the Jacobian of the gradient dynamics is given by
\begin{equation}
D\omega(x)= \begin{bmatrix} a & b \\ c & d \end{bmatrix}
\label{eq:domega2}
\end{equation}
with $a,b,c,d \in \mathbb{R}$. If $x$ is a non-degenerate differential Nash equilibria, $a,d>0$ and
$\det(D\omega(x))\ne 0$ which implies that $ad\ne cb$. Choosing $c,d$ such that
$ad<cb$ will guarantee that one of the eigenvalues of $D\omega(x)$ is negative
and the other is positive, making $x$ a strict saddle point. This shows that non-degenerate differential Nash equilibria can be
strict saddle points of the combined gradient dynamics.

Hence, for any game $(f_1,\ldots, f_n)$, a non-degenerate differential Nash equilibrium is either a locally asymptotically stable equilibrium or a strict saddle point, but
it not strictly unstable or strictly marginally stable (i.e.~having eigenvalues
all on the imaginary axis).
\end{proof}

The proof of Proposition~\ref{prop:zsg}, which claims that all differential Nash equilibria in zero-sum games are locally asymptotically stable, again just relies on basic linear algebra and the definition of a differential Nash equilibrium.

\begin{proof}[Proof of Proposition~\ref{prop:zsg}]
    Consider  a two player game $(\f, -\f)$ on $X_1\times
    X_2=\mb{R}^m$ with
$X_i=\mb{R}^{m_i}$. For such a game,  
\[D\omega(x)=\bmat{\ \ D_{1}^2f(x) &\ \  D_{21} f(x)\\
    -D_{12}f(x) & -D_{2}^2f(x)}.\]
Note that   $D_{21}f(x)=(D_{12}f(x))^T$.
Suppose that $x=(x_1,x_2)$ is a differential Nash equilibrium 
and let $v=[v_1,v_2]\in \mathbb{R}^m$ with 
 $v_1 \in \mb{R}^{m_1}$ and $v_2 \in \mb{R}^{m_2}$. Then,
$v^TD\omega(x)v
= v_1^TD_{1}^2f(x)v_1-v_2^TD_{2}^2f(x)v_2 >0$
 since $D_{1}^2f(x)\succ0$ and $-D_{2}^2f(x)\succ0$ for
$x$, a differential Nash equilibrium. 
Since $v$ is arbitrary, this implies that $D\omega(x)$ is positive definite and
hence, clearly non-degenerate.
Thus, for two-player zero-sum games, all differential Nash equilibria are both
non-degenerate differential Nash equilibria and locally asymptotically stable equilibria of $\dot x=-\omega(x)$
\end{proof}

The proof that all locally asymptotically stable equilibria in potential games are differential Nash equilibria relies on the symmetry of $D\omega$ in potential games.

\begin{proof}[Proof of Proposition~\ref{prop:potentialgame}]
    The proof  follows from the definition of a potential game.
    Since $(f_1,$ $\ldots, f_n)$ is a potential game, it admits a potential
    function $\phi$ such that $D_if_i(x)=
D_i\phi(x)$ for all $x$. This, in turn, implies that at a locally asymptotically
stable equilibrium of $\dot x=-\omega(x)$, $D\omega(x)=D^2\phi(x)$, where
$D^2\phi$ is the Hessian matrix of the function $\phi$. Further $D^2\phi(x)$
must have strictly positive eigenvalues for $x$ to be a locally asymptotically
stable equilibrium of $\dot x=-\omega(x)$. Since the Hessian matrix of a
function must be symmetric, $D^2\phi(x)$, must be positive definite, which
through Sylvester's criterion ensures that each of the diagonal blocks of $D^2\phi(x)$ is positive definite. Thus, we have that the existence of a potential function guarantees that the only locally asymptotically stable equilibria of $\dot x=-\omega(x)$, are differential Nash equilibria.
\end{proof}

\subsection{Proofs for Deterministic Setting}
We now present the proof of Theorem~\ref{thm:fullinfo} and its corollaries.
The proof of relies on the celebrated stable manifold
theorem~\cite[Theorem~III.7]{shub:1978aa},
\cite{smale:1967aa}. Given a map $\phi$, we use the notation $\phi^t=\phi \circ \cdots \circ \phi$ to denote the $t$--times composition of $\phi$.
\begin{theorem}[{Center and Stable Manifolds~\cite[Theorem~III.7]{shub:1978aa},
\cite{smale:1967aa}}]
    Let $x_0$ be a fixed point for the $C^r$ local diffeomorphism $f:U\rar
    \mb{R}^d$ where $U\subset \mb{R}^d$ is an open neighborhood of $x_0$ in
    $\mb{R}^d$ and $r\geq 1$. Let $E^s\oplus E^c\oplus E^u$ be the invariant
    splitting of $\mb{R}^d$ into generalized eigenspaces of $D\phi(x_0)$
    corresponding to eigenvalues of absolute value less than one, equal to one,
    and greater than one. To the $D\phi(x_0)$ invariant subspace $E^s\oplus E^c$
    there is an associated local $\phi$--invariant $C^r$ embedded disc
    $W_{\text{loc}}^{cs}$ called the local stable center manifold of dimension $\dim(E^s\oplus E^c)$ and ball $B$ around
    $x_0$ such that
    $\phi(W_{\text{loc}}^{cs})\cap B\subset W_{\text{loc}}^{cs}$, and if $\phi^t(x)\in
    B$ for all $t\geq 0$, then $x\in W_{\text{loc}}^{sc}$.
    \label{thm:centerstable}
\end{theorem}
Some parts of the
proof follow similar arguments to the proofs of results
in~\cite{lee:2016aa,panageas:2016aa} which apply to (single-agent)
gradient-based optimization. Due to the different learning rates employed by the agents and
 the introduction of the differential game form $\omega$, the proof
 differs. 

\begin{proof}[Proof of Theorem~\ref{thm:fullinfo}]
The proof is composed of two parts: (a) the map $g$ is a diffeomorphism, and (b) application of the stable manifold theorem to conclude that the set of initial conditions is measure zero.

\paragraph{(a) $g$ is diffeomorphism}  We claim the mapping $g:\mb{R}^m\rar \mb{R}^m$ is a diffeomorphism. If we can
 show that $g$ is invertible and a local diffeomorphism, then the claim follows. Consider $x\neq y$ and suppose $g(y)=g(x)$ so that
    $y-x=\gamma\cdot (\omega(y)-\omega(x))$. The assumption $\sup_{x\in
        \mb{R}^m}\|D\omega(x)\|_2\leq L<\infty$ implies that $\omega$ satisfies
        the Lipschitz condition on $\mb{R}^m$. Hence,
        $\|\omega(y)-\omega(x)\|_2\leq L\|y-x\|_2$.
        Let $\Gamma=\diag({\Gamma}_1, \ldots, {\Gamma}_n)$ where
        ${\Gamma}_i=\diag( (\gamma_i)_{j=1}^{m_i})$---that is,
        ${\Gamma}_i$ is an $m_i\times m_i$ diagonal matrix with $\gamma_i$
        repeated on the diagonal $m_i$ times.
       Then, 
       $\|x-y\|_2\leq L\|\Gamma\|_2\|y-x\|_2<\|y-x\|_2$
       since $\|\Gamma\|_2=\max_i|\gamma_i|<1/L$.

       Now, observe that $Dg=I-\Gamma D\omega(x)$. If $Dg$ is invertible, then the
        implicit function theorem~\cite[Theorem~C.40]{lee:2012aa} implies that
        $g$ is a local diffeomorphism. Hence, it suffices to show that $\Gamma
        D\omega(x)$ does not have an eigenvalue of $1$. Indeed, letting $\rho(A)$ be
        the spectral radius of a matrix $A$, we know in general that
        $\rho(A)\leq \|A\|$ for any square matrix $A$ and induced operator norm
        $\|\cdot\|$ so that
         $\rho(\Gamma D\omega(x))\leq \|\Gamma D\omega(x)\|_2\leq
        \|\Gamma\|_2\sup_{x\in\mb{R}^m}\|D\omega(x)\|_2<\max_i|\gamma_i|L<1$
        Of course, the spectral radius is the maximum absolute value of the
        eigenvalues, so that the above implies that all eigenvalues of  $\Gamma
        D\omega(x)$ have absolute value less than $1$.

 Since $g$ is injective by the preceding argument, its inverse is well-defined and since $g$
        is a local diffeomorphism on $\mb{R}^m$, it follows that $g^{-1}$ is
        smooth on $\mb{R}^m$. Thus, $g$ is a diffeomorphism.

   \paragraph{(b) Application of the stable manifold theorem} Consider all critical points to the game---i.e.~$\mc{X}_c=\{x\in
            {X}|\ \omega(x)=0\}$. For each $p\in\mc{X}_c$, let $B_p$ be the
            open ball derived from Theorem~\ref{thm:centerstable} and let
            $\mc{B}=\cup_{p}B_p$. Since ${X}\subseteq \mb{R}^{m}$, Lindel\~{o}f's
            lemma~\cite{kelley:1955aa}---every open cover has a countable subcover---gives a
            countable subcover of $\mc{B}$. That is, for a countable set of
            critical points $\{p_i\}_{i=1}^\infty$ with $p_i\in \mc{X}_c$, we
            have that $\mc{B}=\cup_{i=1}^\infty B_{p_i}$. 

            Starting from some point $x_0\in X$, if gradient-based learning
             converges to a strict saddle point, then there exists a $t_0$ and
             index $i$ such that $g^t(x_0)\in B_{p_i}$ for all $t\geq t_0$.
             Again, applying Theorem~\ref{thm:centerstable} and using that
             $g(X)\subset X$---which we note is obviously true if
             $X=\mb{R}^m$---we get that
             $g^t(x_0)\in W_{\text{loc}}^{cs}\cap X$. 

             Using the fact that $g$ is invertible, we can iteratively construct
             the sequence of sets defined by
             $W_1(p_i)=g^{-1}(W_{\text{loc}}^{cs}\cap X)$ and
             $W_{k+1}(p_i)=g^{-1}(W_k(p_i)\cap X)$. Then we have that $x_0\in
             W_t(p_i)$ for all $t\geq t_0$. The set $\mc{X}_0=\cup_{i=1}^\infty
             \cup_{t=0}^\infty W_t(p_i)$ contains all the initial points in $X$
             such that gradient-based learning converges to a strict saddle. 
     Since $p_i$ is a strict saddle, $I-\Gamma D\omega(p_i)$ has an
             eigenvalue greater than $1$. This implies that the
             co-dimension of $E^u$ is strictly less than $m$.
             (i.e.~$\dim(W_{\text{loc}}^{cs})<m$). Hence,
             $W_{\text{loc}}^{cs}\cap X$ has Lebesgue measure zero in
             $\mb{R}^m$. 

             Using again that $g$ is a diffeomorphism, $g^{-1}\in C^1$ so that
             it is locally Lipschitz and locally Lipschitz maps are null set
             preserving. Hence, $W_k(p_i)$ has measure zero for
             all $k$ by induction so that $\mc{X}_0$ is a measure zero set since
             it is a countable union of
             measure zero sets.
\end{proof}

The proof of Corollary~\ref{cor:msfinite} follows from the symmetry of $D\omega$ in potential games, and our observations in Section~\ref{sec:connections}.

\begin{proof}[Proof of Corollary~\ref{cor:msfinite}]
Since the game admits a potential function $\phi$, there is a transformation of
coordinates such that agents following the dynamics
$x_{t+1}=x_t-\gamma\odot\omega(x_t)$ converge to the same equilibria as the
gradient dynamics $x_{t+1}=x_t-\gamma\odot D\phi(x_t)$.
Hence, the analysis of the gradient-based learning scheme reduces to analyzing 
gradient-based optimization of
$\phi$. 
Moreover, existence of a potential function also implies that $D_{ij}f_j\equiv
D_{ji}f_i$ so that $D\omega$ is symmetric.
Indeed, writing $\omega(x)$ as the differential form $\sum_{i=1}^n
D_if_i(x) \dd x_i$
and noting that 
 $\dd \circ \dd=0$ for the differential operator $d$, we have that 
    $\dd(\omega)=\sum_i\dd(D_if_i )\wedge %
        \dd x_i=\sum_{i,j: j>i}\left(D_{ij}f_j-D_{ji}f_i \right)\dd x_i\wedge
        \dd x_j=0$ where $\wedge$ is the standard exterior product~\cite{lee:2012aa}.
Symmetry of $D\omega$ implies that all periodic orbits are
equilibria---i.e.~the dynamics do not possess any limit cycles.
By Theorem~\ref{thm:fullinfo}, the set of initial points that converge to
strict saddle points is of measure zero.    Since all the stable
critical points of the dynamics are equilibria, with the assumption that
$\lim_{t\rar \infty} g^t(x)$ exists for all $x\in X$, we have that
    $P_\nu\left[ \lim_{t\rar \infty}g^t(x)=x^\ast \right]=1$
where $x^\ast$ is 
a non-degenerate differential Nash equilibrium which is generically a local Nash
equilibrium~\cite{ratliff:2014aa}.

\end{proof}

\subsection{Classical Results from Dynamical Systems}
The remaining results use the following classical result from dynamical
systems theory.  
 Consider a general stochastic approximation framework
$x_{t+1}=x_t+\gamma_t(h(x_t))+\epsilon_t$
for $h:X\rar TX$ with $h\in C^2$ and where $X\subset \mb{R}^d$ and where $TX$
denotes the tangent space. 
\begin{theorem}[{Theorem~1~\cite{pemantle:1990aa}}]
    Suppose $\gamma_t$ is $\mc{F}_t$--measurable and $\mb{E}[w_t|\mc{F}_t]=0$.
    Let the stochastic process $\{x_t\}_{t\geq 0}$ be defined as above for some
    sequence of random variables $\{\epsilon_t\}$ and $\{\gamma_t\}$. Let $p\in X$
    with $h(p)=0$ and let $W$ be a neighborhood of $p$. Assume that there are
constants $\eta\in(1/2,1]$ and $c_1,c_2,c_3,c_4>0$ for which the following
conditions are satisfied whenever $x_t\in W$ and $t$ sufficiently large:
(i) $p$ is a linear unstable critical point,
(ii) $c_1/t^{\eta}\leq \gamma_t\leq c_2/t^{\eta}$,
 (iii) $\mb{E}[(w_t\cdot v)^+|\mc{F}_t]\geq c_3/t^{\eta}$ for every unit
        vector $v\in TX$,
and (iv) $\|w_t\|_2\leq c_4/t^{\eta}$.
Then $P(x_t\rar p)=0$.
    \label{thm:pementle}
\end{theorem}

\ifms
\else
\section{Expanded Results in the Stochastic Setting} In this appendix , we provide extended results in the stochastic setting that require more mathematical formalism than the main body of the paper. In addition, we introduce a new class of games that generalize potential games and have stronger convergence guarantees than the broader class of general-sum continuous games. 
\label{app:ms}
\subsection{Avoidance of Repelling Sets}
\label{app:repel}
To show that stochastic gradient-based learning avoids of more general limiting behaviors than saddle points, we need further assumptions on our underlying space---i.e.~we need the
underlying decision spaces of each agent---i.e.~$X_i$ for each $i\in \mc{I}$---to be \emph{smooth, compact manifolds without
boundary}\footnote{The torus $\mb{T}=\mb{S}^1\times \mb{S}^1$ is an example.  The interested reader can consult, e.g.,~\cite{lee:2012aa} for more details on differential geometry.}. 
The stochastic process $\{x_n\}$ which follows \eqref{eq:sa} is \emph{defined
on} $X$---that is, $x_n\in X$ for all $n\geq 0$. As before, it is natural to
compare sample points $\{x_n\}$ to solutions of $\dot{x}=-\omega(x)$ where we
think of \eqref{eq:sa} as a noisy approximation.  The asymptotic behavior of
$\{x_n\}$ can indeed be described by the asymptotic behavior of the flow
generated by $\omega$. 

We also need a formal notion of \emph{cycles}. A non-stationary periodic orbit of $\omega$ is called a \emph{cycle}. Let
$\xi\subset X$ be a cycle of period $T>0$. Denote by $\Phi_T$ the flow
corresponding to $\omega$. For any $x\in \xi$, $\spec(D\Phi_T(x))=\{1\}\cup
C(\xi)$ where $C(\xi)$ is the set of characteristic multipliers. We say $\xi$ is
\emph{hyperbolic} if no element of $C(\xi)$ is on the complex unit circle.
Further, if $C(\xi)$ is strictly inside the unit circle, $\xi$ is called \emph{linearly
stable} and, on the other hand, if $C(\xi)$ has at least one element on the
outside of the unit circle---that is, $D\Phi_T(x)$ for $x\in \xi$ has an
eigenvalue with real part strictly greater than $1$---then $\xi$ is called
\emph{linearly unstable}. The latter is the analog of strict saddle points in the context of periodic orbits.
We denote by $\{x_t\}$ sample paths
of the process \eqref{eq:sa} and $L(\{x_t\})$ is the \emph{limit set} of any
sequence $\{x_t\}_{t\geq 0}$ which is defined in the usual way as
all $p\in X$ such that $\lim_{k\rar \infty} x_{t_k}=p$ for some
sequence
$t_k\rar \infty$.
It was shown in~\cite{benaim:1996aa} that under less restrictive assumptions
than Assumptions~\ref{ass:estmartin} and \ref{ass:others}, $L(\{x_t\})$ is
contained in the
\emph{chain recurrent set} of $\omega$ and $L(\{x_t\})$ is a non-empty, compact
and connected set invariant under the flow of $\omega$.
\begin{theorem}
    Consider a game $(f_1,\ldots, f_n)$ where each $X_i$ is a smooth, compact
    manifold without boundary.   Suppose each agent $i\in \mc{I}$ adopts a
    stochastic gradient-based
   learning algorithm that satisfies Assumptions~\ref{ass:estmartin}
   and \ref{ass:others} and is such that sample points $x_t\in X$ for all $t\geq
   0$. Further, suppose that for each $i\in \mc{I}$, there
   exist a constant $b_i>0$ such that 
   $\mb{E}[(w_{i,t}\cdot v)^+|\mc{F}_{i,t}]\geq b_i$
   for every unit vector $v\in \mb{R}^{m_i}$. Then competitive stochastic
   gradient-based learning converges to 
  linearly unstable cycles on a set of measure
   zero---i.e.
   $P(L(x_t)=\xi)=0$
    where $\{x_t\}$ is a sample path.
    \label{thm:gradfreecycle}
\end{theorem}
As we noted, periodic orbits are not necessarily excluded from the limiting
behavior of gradient-based learning in games. 
We leave out the proof of Theorem~\ref{thm:gradfreecycle} since after some
algebraic manipulation, it is a direct application
of~\cite[Theorem~2.1]{benaim:1995aa} which is 
stated below.
\begin{theorem}[Theorem~2.1~\cite{benaim:1995aa}]
    Let $\xi\subset X$ be a hyperbolic linearly unstable cycle of $h$.
    Assume the following (i) $h\in C^2$; (ii) $c_1/t^{\eta}\leq \gamma_t\leq c_2/t^{\eta}$ with
    $0<c_1\leq c_2$ and $0<\eta\leq 1$; and (iii) there exists $b\geq 0$ such that
    for all unit vectors $v\in \mb{R}^m$, $\mb{E}[(w_t\cdot
    v)^+|\mc{F}_t]\geq b$. Then $P(L(\{x_t\})=\xi)=0$.
    \label{thm:benaim}
\end{theorem}

 \subsection{Morse-Smale Games}
 \label{app:msms}
For a class of games admitting \emph{gradient-like} vector fields we can go beyond non-convergence results and give convergence guarantees. 
Following~\cite{benaim:1995aa}, we introduce a new class of games, which we call \emph{Morse-Smale
games}, that are a generalization of potential games. Such games represent an important
class since the vector field of $\omega$ corresponds to Morse-Smale vector field which is known
to be generic in $\mb{R}^2$ and are otherwise structurally
stable~\cite{hirsch:1976aa,palis:1970aa}. 

\begin{definition}
    A game $(f_1,\ldots, f_n)$ with $f_i\in C^r$ for some $r\geq 3$ and where strategy spaces $X_i$ is a smooth,
    compact manifold without boundary for each $i\in \mc{I}$ is a Morse-Smale game if the vector field
    corresponding to the differential $\omega$ is Morse-Smale---that is, the
    following hold:
    (i) all periodic orbits $\xi$ (i.e.~equilibria and cycles) are
            hyperbolic and $W^s(\xi)\pitchfork W^u(\xi)$ (i.e.~the stable and
            unstable manifolds of $\xi$ intersect transversally),
      (ii) every forward and backward omega limit set is a periodic orbit,
     (iii) and $\omega$ has a global attractor.
\end{definition}
The conditions of Morse-Smale in the above definition ensure that there are only
finitely many periodic orbits.  The dynamics of games with more general
vector fields, on the other hand, can admit chaos (e.g. the classic Lorentz attractor can be cast as gradient-play in a 3-player game). Hyperbolic equilibria and
periodic orbits are the only types of limiting
behavior that have been shown to correspond to strategies relevant to the
underlying game~\cite{benaim:1997ab}.
The simplest example of a Morse-Smale vector field is a gradient flow. However,
not all Morse-Smale vector fields are gradient flows and hence, not all
Morse-Smale games are potential games. 
\begin{example}
    Consider the
$n$-player game with $X_i=\mb{R}$ for each $i\in \mc{I}$ and
$f_n(x)=x_n(x_1^2-1), \ f_i(x)=x_ix_{i+1}, \ \forall i\in \mc{I}/\{n\}$
This  is a Morse-Smale game that is not a potential game.  Indeed, $\dot{x}=-\omega(x)$ where
$\omega=[x_2, x_3, \ldots, x_{n-1}, x_1^2-1]$
is a dynamical system with a Morse-Smale vector field that is not a gradient
vector field~\cite{conley:1978aa}.
\end{example}

Essentially, in a neighborhood of a critical point for a Morse-Smale game, the
game behavior can be described by a Morse function $\phi$ such that near critical
points $\omega$ can be written as $D\phi$ and away from critical points $\omega$
points in the same direction as $D\phi$---i.e.~$\omega\cdot D\phi>0$. 
Specializing the class of Morse-Smale games, we have stronger convergence
guarantees. 
\begin{theorem}
    Consider a Morse-Smale game $(f_1,\ldots, f_n)$ on smooth boundaryless
    compact manifold $X$. Suppose
    Assumptions~\ref{ass:estmartin} and \ref{ass:others} hold and that $\{x_t\}$
    is defined on $X$. Let $\{\xi_i, \
    i=1, \ldots, l\}$ denote the set of periodic orbits in $X$. Then
    $\sum_{i=1}^l P(L(\{x_t\})=\xi_i)=1$ and $P(L(\{x_t\})=\xi_i)>0$
    implies $\xi_i$ is linearly stable.
    Moreover, if the periodic orbit $\xi_i$
    with $P(L(\{x_t\})=\xi_i)>0$
     is an equilibrium, then it is either a
    non-degenerate differential Nash equilibrium---which is generically a local
    Nash---or a non-Nash locally asymptotically stable equilibrium. 
    \label{thm:Morsesmale}
\end{theorem}
 The proof of Theorem~\ref{thm:Morsesmale} follows by invoking
 Corollary~\ref{cor:benaim} which is stated below.
 \begin{corollary}[Corollary~2.2~\cite{benaim:1995aa}]
    Assume that there exists $\delta \geq 1$ such that
    $\sum_{n\geq 0}\gamma_n^{1+\delta}<\infty$ and that $h$ is a Morse-Smale vector
    field. If we denote by $\{\xi_i, \ i=1, \ldots, l\}$ the set of
    periodic orbits in $X$, then $\sum_{i=1}^lP(L(\{x_t\})=\xi_i)=1$.
    Further, if conditions (i)--(iii) of Theorem~\ref{thm:benaim} hold, then
    $P(L(\{x_t\})=\xi_i)>0$ implies $\xi_i$ is linearly stable.
    \label{cor:benaim}
\end{corollary}

Thus, in  Morse-Smale games, with probability one, the limit sets of competitive  gradient-based learning with stochastic updates are attractors (i.e., periodic orbits, which includes limit cyles and equilibria) of $\dot{x}=-\omega(x)$ and if any attractor has positive probability of being a limit set of the players' collective update rule, then it is (linearly) stable.
Moreover, attractors that are equilibria are either non-degenerate differential Nash equilibria (generically local Nash equilbiria) or non-Nash locally asymptotically stable equilibria, but not saddle points. 

If we further restrict the class of games to potential games, the results for Morse-Smale games
imply convergence to Nash almost surely, a particularly strong convergence guarantee. 
\begin{corollary}
  Consider the game $(f_1,\ldots, f_n)$ on smooth boundaryless
    compact manifold $X=X_1\times \cdots \times X_n$ admitting potential function $\phi$. Suppose each agent $i\in \mc{I}$ adopts a
    stochastic gradient-based
  learning algorithm that satisfies Assumptions~\ref{ass:estmartin}
  and \ref{ass:others} and such that $\{x_t\}$ evolves on $X$. Further, suppose that for each $i\in \mc{I}$, there
  exist a constant $b_i>0$ such that 
  $\mb{E}[(w_{i,t}\cdot v)^+|\mc{F}_{i,t}]\geq b_i$
  for every unit vector $v\in \mb{R}^{m_i}$.
  Then, competitive stochastic gradient-based learning converges
  to a non-degenerate differential Nash equilibrium almost surely.
  \label{cor:potgradfree}
\end{corollary}

The proof of Corollary~\ref{cor:potgradfree} follows from the fact that potential games are trivially Morse-Smale games that admit no periodic cycles as we showed in the proof of Corollary~\ref{cor:msfinite}. 

\begin{proof}[Proof of Corollary~\ref{cor:potgradfree}]
    Consider a potential game $(f_1,\ldots, f_n)$ where each $X_i$ is a smooth,
    compact boundaryless manifold. Then $\omega=D\phi$ for some $\phi\in C^r$
    which implies that $\omega$ is a gradient flow and hence, does not admit
    limit cycles. Let $\{\xi_i, \ i=1, \ldots, l\}$ be the set of equilibrium
    points in $X$. Under the assumptions of Theorem~\ref{thm:Morsesmale},
     $\sum_{i=1}^lP(L(\{x_t\})=\xi_i)=1$ and, if
    $P(L(\{x_t\})=\xi_i)>0$, then $\xi_i$ is a linearly stable equilibrium point
    which is a non-degenerate differential Nash equilibrium of the game due to
    the fact that $D\omega(x)$ is symmetric in potential games.
    Hence,  a sample path $\{x_t\}$ converges to a non-degenerate differential Nash equilibrium with probability one. Moreover, 
    by \cite{ratliff:2014aa}, we know it is generically a local Nash.
\end{proof}

We note, that even though a potential function is
enough to guarantee convergence to a local Nash equilibrium, potential games can
still admit local Nash equilibria that are strict saddle points as shown in Section~\ref{sec:connections}. Thus, even this
relatively well-behaved class of games has fundamental problems when applying a
gradient-based learning scheme. \fi
\section{Classes of Gradient-Based Learning Algorithms}
\label{sec:examples}
In this section, we provide derivation of the gradient-based learning rules
provided in Table~\ref{tab:examples}.
We note that the derivation of gradient-based approaches for multi-armed bandits can be found
in~\cite{sutton:2017aa} among other classic references on reinforcement
learning.  
\subsection{Online Optimization: Gradient Play in Non-Cooperative Games}
\label{sec:limitcycle}

We first show that classical online optimization algorithms fit into the framework we describe. In this case, each agent is directly trying to minimize their own function $f_i(x_i,x_{-i})$, which can depend on the current iterate of the other agents. 
There are many examples in the optimization literature of this type of setup.
We note that in the full information case, the competitive gradient-based
learning framework we describe here is simply \emph{gradient
play}~\cite{fudenberg:1998aa}, a very well-studied game-theoretic learning rule.

Of more interest are some gradient-free online optimization algorithms that also fit into the framework we describe. The game can be described as follows.
At each iteration, $t$ of the game, every player publishes their current iterate
$x_{i,t}$. Player $i$, implementing this algorithm, then updates their iterate by
taking a random unit vector $u$, and querying $f_i(x_i+\delta_iu,x_{-i})$. The
update map is given by
$x_{i,t+1}=x_{i,t}-\gamma_i f_i(x_i+\delta_iu,x_{-i})u.$
It is shown in~\cite{flaxman:2005aa} that $f_i(x_i+\delta_iu,x_{-i})u$ is an
unbiased estimate of the gradient of a smoothed version of
$f_i$---i.e.
$\hat{f}_i(x_i,x_{-i})=\mathbb{E}_v[f_i(x+\delta v,x_{-i})]$.
 Thus the loss function being minimized by the agent is $\hat{f}_i$. In this
 case, the results on characterizing limiting behavior presented in Section~\ref{sec:gradientfree} apply.
\subsection{Generative Adversarial Networks}
Generative adversarial networks take a game theoretic approach to fitting a
generative model in complex structured spaces. Specifically, they approach the
problem of fitting a generative model from a data set of samples from some
distribution $Q\in \Delta(Y)$ as a zero-sum game between a \emph{generator} and a
\emph{discriminator}.
In general, both the generator and the discriminator are modeled as deep neural networks.
The generator network outputs a
sample $G_\theta(z)\in Y$ in the same space $Y$ as the sampled data set given a random noise signal
$z\sim F$  as
an input. The discriminator $D_w(y)$ tries  to  discriminate between
a true sample and a sample generated by the generator---that is, it takes as
input a sample $y$ drawn from $Q$ or the generator and tries to determine if its
\emph{real} or \emph{fake}. 
The goal, is to find a Nash equilibrium of the zero-sum game under which the generator will learn to
generate samples that are  indistinguishable
from the true samples---i.e.~in equilibrium, the generator has learned the
underlying distribution.

To prevent instabilities in the training of GANs with zero-one discriminators,
the Wasserstein GAN attempts to approximate the Wasserstein-1 metric between the true
distribution and the distribution of the generator. In this setting,
$D_w(\cdot)$ is a $1$-Lipschitz function leading to the problem
\begin{align*}
 \textstyle   \inf_\theta\sup_w \mb{E}_{y\sim Q}[D_w(y)]-\mb{E}_{z\sim
    F}[D_w(G_\theta(z))]
\end{align*}
which has corresponding dynamics  $w_{t+1}=w_t+\gamma \nabla_wL(\theta_t,w_t)$
and $\theta_{t+1}=\theta_t-\gamma \nabla_\theta L(\theta_t,w_t)$ where 
   $ L(\theta, w)= \mb{E}_{y\sim Q}[D_w(y)]-\mb{E}_{z\sim
    F}[D_w(G_\theta(z))]$ and
where $\gamma$ is the learning rate.

GANs are notoriously difficult to train. The typical approach is to allow each
player to perform (stochastic) gradient descent on the derivative of
their cost with respect to their own choice variable. 
There are two important observations about gradient-based learning approaches to
GANs relevant to this paper. First, the equilibrium that is sought is generally
a saddle point and second, the dynamics of GANs are complex enough to admit
limit cycles ~\cite{mertikopoulos:2018aa}.
None-the-less, training GANs with gradient descent is still
very common.  We note that our results suggest that, on top of periodic orbits and oscillations, training GANs with gradient descent can result in convergence to non-Nash equilibria.

\subsection{Multi-Agent Reinforcement Learning Algorithms}

Consider a setting in which all agents are operating in an MDP. There is a shared state space $\mathcal{S}$. 
Each agent, indexed by $\mathcal{I}=\{1,\ldots,n\}$ has their own action space
$U_{i}$ and reward function $R_i: \mathcal{S} \times  U \rar \Delta_\mathbb{R}$
where $U=U_1\times \cdots\times U_n$. We note the reward functions could themselves be random, but for illustrative purposes we suppose they are deterministic.
Finally, the dynamics of the MDP are described by a state transition kernel
$P:\mathcal{S} \times  U \rar \Delta_\mathcal{S}$ and an initial state distribution $P_0$. 
Each agent $i$ also has a policy, $\pi_{i}$, that returns a distribution over $U_i$ for each state $s \in \mathcal{S}$. 
We define a trajectory of the MDP, $\tau$ as
$\tau=\{(s_t,u_{i,t},u_{-i,u})\}^{T-1}_{t=0}$. Thus, a trajectory is a finite
sequence of states, the actions of each player in that state, and the reward
agent $i$ received in that state, where $T$ is the time horizon. Given fixed
policies we can define a distribution over the space of all trajectories
$\Gamma$, namely $P_{\Gamma}(\pi)$, by 
\begin{align*}
  \textstyle  P_{\Gamma}(\tau;\pi)&  \textstyle=P_0(s_0)\prod_{i\in
\mathcal{I}}\pi_i(u_{i,0}|s_0)\cdots P(s_{t}|s_{t-1},u_{t-1})\prod_{i\in
\mathcal{I}}\pi_i(u_{i,{t}}|s_{t})\cdots
\end{align*}
The goal of each single agent in this setup is to maximize their cumulative
expected reward over a time horizon $T$. That is, the agent is trying to find a policy $\pi_i$ so as to maximize some function, which in keeping with our general formulation in Section \ref{sec:prelims}, we write as $-f_i$ since this problem is a maximization.
When an agent is employing policy gradient in this MARL setup, we assume that
their policy comes from a parametric class of policies parametrized by $x_i \in
X_i\subset \mathbb{R}^{m_i}$. To simplify notation, we write the parametric policy as
$\pi_i(x_i)$ where for each $x_i$, given an state $s$, $\pi_i(x_i)$ is a
probability distribution on actions $u_i$ which we denote by
$\pi_i(x_i)(\cdot|s)$. 

The policy gradient MARL algorithm can be reformulated in the
    competitive gradient-based learning framework.      
    An agent $i$ using policy gradient is trying to tune the parameters $x_i$ of
their policy to maximize their expected reward over a trajectory of length
$T$. We define the reward of agent $i$ over a trajectory of the MDP, $\tau\in
\Gamma$, to be $\boldsymbol{R}_i(\tau)=\sum_{t=0}^{T-1} R_i(s_t,_{i,t},u_{-i,t})$.
Thus, each agent's loss function $f_i$, in keeping with our notation, is given by
$f_i(x_i,x_{-i})=-J_i(\pi_i(x_i),\pi_{-i}) = -\mathbb{E}_{\tau \sim
    P_{\Gamma}(\pi)}[\boldsymbol{R}_i(\tau))].$
The actions of agent $i$ in the continuous
game framework described in previous sections are the parameters of their
policy, and thus their action space is $X_i\subset\mathbb{R}^{m_i}$. We note that we have
made no assumptions on the other player's actions $x_{-i}$. That is, they do not
need to be employing the same parameterized policy class or exactly the same
gradient-based update procedure; the only requirement is that they also be using a gradient based multi-agent learning algorithm, and that their actions give rise to a set of policies $\pi_{-i}$ that govern the way they choose their actions in the MDP.

In the full information case, at each round, $t$ of the game, a player plays
according to $\pi_{i}(x_{i,t})$ for a time horizon $T$, and then performs a
gradient update on their parameters where $D_if_i(x_i,x_{-i})=D_iJ_i(\pi_i({x_i}),\pi_{-i,t})$ is given by
\begin{align}
D_i J_i(\pi_i({x_i}),\pi_{-i})&=\textstyle\mathbb{E}_{\tau\sim
    P_{\Gamma}(\pi)}\Big[
\sum_{t=0}^{T-1}R_i(s_t,u_t)\textstyle\sum_{j=0}^t\nabla_{x_i}
\log\pi_i(x_i)(u_{i,j}|s_{j})\Big]
\label{eq:polgrad}
\end{align}
The derivation of this gradient is exactly the same as that of classic policy gradient.
From \eqref{eq:polgrad} it is clear that an unbiased estimate of the gradient can be constructed. At each time $t$ in the
policy gradient update procedure, agent $i$ receives a $T$ horizon roll-out, say
$z_{i,t}=\{(s_{k},u_{i,k},r_{i,k})\}_{k=0}^{T-1}$, and constructs the
unbiased estimate of the gradient---i.e.
$\textstyle\widehat{\D_{i} J_i}=\sum_{k=0}^{T-1}r_{i,k}\big(\sum_{j=0}^k\nabla_{x_i}
\log\pi_i(x_{i,t})(u_{i,j}|s_{j})\big).$
We note that in this case, the agent does not need to know the policies of the
other agents, or anything about the dynamics of the MDP. The agent can construct the estimator solely from the sequence of states, the reward they received in those states, and their own actions.
With these two derivations of the gradient for the full information and
gradient-free cases, policy gradient for MARL conforms to the competitive
gradient-based learning framework and hence, the results of
Section~\ref{sec:results} apply under appropriate assumptions.

\bibliographystyle{siamplain}
\bibliography{arxiv}

\end{document}